\documentclass{article}

\PassOptionsToPackage{round}{natbib}
\usepackage[preprint]{neurips_2026}

\usepackage[utf8]{inputenc}
\usepackage[T1]{fontenc}

\usepackage{hyperref}
\usepackage{url}
\usepackage{booktabs}
\usepackage{amsfonts}
\usepackage{amsmath,amssymb,amsthm}
\usepackage{aliascnt}
\usepackage{nicefrac}
\usepackage{microtype}
\usepackage{xcolor}
\usepackage{graphicx}
\usepackage{enumitem}

\newtheorem{theorem}{Theorem}[section]
\newaliascnt{proposition}{theorem}
\newtheorem{proposition}[proposition]{Proposition}
\aliascntresetthe{proposition}
\newaliascnt{corollary}{theorem}
\newtheorem{corollary}[corollary]{Corollary}
\aliascntresetthe{corollary}
\newaliascnt{lemma}{theorem}

\aliascntresetthe{lemma}
\newaliascnt{definition}{theorem}
\newtheorem{definition}[definition]{Definition}
\aliascntresetthe{definition}

\newaliascnt{assumption}{theorem}

\aliascntresetthe{assumption}
\newaliascnt{remark}{theorem}
\newtheorem{remark}[remark]{Remark}
\aliascntresetthe{remark}
\newaliascnt{example}{theorem}
\newtheorem{example}[example]{Example}
\aliascntresetthe{example}
\usepackage{cleveref}
\crefname{proposition}{Proposition}{Propositions}
\Crefname{proposition}{Proposition}{Propositions}
\crefname{corollary}{Corollary}{Corollaries}
\Crefname{corollary}{Corollary}{Corollaries}
\crefname{lemma}{Lemma}{Lemmas}
\Crefname{lemma}{Lemma}{Lemmas}
\crefname{definition}{Definition}{Definitions}
\Crefname{definition}{Definition}{Definitions}
\crefname{remark}{Remark}{Remarks}
\Crefname{remark}{Remark}{Remarks}
\crefname{example}{Example}{Examples}
\Crefname{example}{Example}{Examples}
\crefname{axiom}{Axiom}{Axioms}
\Crefname{axiom}{Axiom}{Axioms}
\crefname{assumption}{Assumption}{Assumptions}
\Crefname{assumption}{Assumption}{Assumptions}

\newcommand{\R}{\mathbb{R}}

\newcommand{\GL}{\mathrm{GL}}

\begin{document}

\title{Labeled Incidence Structures for Native Transformer Modeling of Text, Knowledge Graphs, and Hypergraphs}

\author{%
  Mahesh Godavarti \\
  A Carrot, Inc. \\
  \texttt{m@acarrot.com}
}

\maketitle

\begin{abstract}
Current Transformer interfaces index tokens by one or more integer coordinates, which determine their addresses inside attention.
In RoPE and its multi-axis or hierarchical variants, the resulting address has the form $A(i)=R_1^{i_1}R_2^{i_2}R_3^{i_3}$, where the exponents are integer coordinates assigned after choosing a serialized token layout.
When Transformers process new or large collections of data, this addressing scheme can produce unseen offsets or coordinate combinations, push repositories toward retrieve-and-serialize pipelines, and force new entities, records, or repository items to be represented by long token strings or identifier embeddings not seen in training.

We introduce \emph{labeled incidence structures} (LIS), in which each participating token or value is an endpoint with content $x$ and a structural index $i$.
The index can include local position, relation role, relation instance, text unit, field, or content-derived identity.
The model maps this index to a structural address $A(i)$, so adding new tokens, facts, text units, or repository items applies the same learned address rule to structural and content coordinates rather than requiring larger integer coordinates, unseen coordinate combinations, or new identifier embeddings.
Attention scores endpoints $i,j$ using $q_i^\top P_{j\to i}k_j$, where journey consistency forces $P_{j\to i}=A(i)^{-1}A(j)$.
When $i$ has several coordinates, such as position, role, and instance, coordinate independence is equivalent to factoring $A(i)$ into one address factor per coordinate.
This recovers RoPE, RoPE-2D, and HiRoPE as special cases.
This allows knowledge-graph (KG) roles, fact instances, and text units to enter the attention score directly.

In controlled shallow diagnostics, the LIS address interface is implemented inside ordinary Transformer attention and yields promising results across text, KG, and $n$-ary settings.
\end{abstract}


\section{Introduction}\label{sec:intro}

\noindent\textbf{Foundation models need a way to take structured inputs directly.}
Foundation-model corpora already include structured sources: code, tables, web markup, knowledge graphs (KGs) such as Wikidata, biomedical resources such as UMLS and DrugBank, JSON records, citation graphs, product catalogs, logs, and scientific records.
These data have roles, fields, tree structure, and shared instance identity, but current practice often flattens them into ordinary token sequences before attention: triples are verbalized, tables are linearized, $n$-ary records are serialized, and structural identifiers are added as embeddings or biases.
This paper argues that roles, relation instances, and positions should enter the attention score directly, rather than being flattened first and then recovered from token content and serialized position.

\noindent\textbf{The scale problem: positional coordinates grow with context.}
Integer coordinates assigned to tokenized inputs grow with context, whether they are 1D, multi-axis, or hierarchical.
Growth then shows up as larger offsets, longer identifiers, or the need to retrieve only a subset.
Positional-coordinate workarounds such as Attention with Linear Biases (ALiBi), position interpolation, YaRN, and randomized positional encodings~\citep{press2022alibi,chen2023positioninterpolation,peng2023yarn,ruoss2023randomized}, as well as position resets, masking, and retrieve-and-serialize pipelines, all try to keep useful interactions inside the trained range or make unseen offsets less harmful: by downweighting, rescaling, randomizing, resetting positions within segments, or retrieving into a short context.
\Cref{sec:scale} treats these as one bottleneck: $P$ is computed from the coordinates of a serialized layout, so useful information must stay within the coordinate ranges and combinations seen during training.

\noindent\textbf{The structural problem: facts are forced through token serialization.}
Structured records contain more than values.
They also specify which role each value plays and which values belong to the same relation instance.
In a KG triple $(\text{Paris},\; \texttt{capital\_of},\; \text{France})$, Paris and France are not just two nearby tokens: they occupy the \textsc{head} and \textsc{tail} roles of one \texttt{capital\_of} fact.
A sentence ``Paris is the capital of France'' gives the same entities positions in a tokenized string.
An $n$-ary fact, such as a drug--gene--disease interaction $(\text{Aspirin},\;\text{COX-2},\;\text{Inflammation})$, also has a shared relation instance: the \textsc{drug}, \textsc{target}, and \textsc{disease} participants co-occur in one event.

Common ways of incorporating structured data into text Transformers so it can be accessed through natural language, including role-marked strings, additive tags, scalar score biases, learned identifier embeddings, and binary reductions of $n$-ary facts, place role and instance information in token content, embeddings, value messages, or scalar bias terms, but not in the pairwise operator $P$ that multiplies the key. Appendix~\ref{app:additive-embeddings} makes this restriction precise.

\noindent\textbf{One representation for text, KGs, and hyperedges.}
We introduce labeled incidence structures (LIS), a representation in which each endpoint keeps its content, its role, and its relation instance, and in which these three parts are turned directly into the address used by the attention score, without passing through a serialized layout.
Text tokens, KG endpoints, and $n$-ary hyperedge endpoints can all be written in this form.
Under the coordinate-independence conditions stated later, ordinary text position can also be included as its own factor, giving forms such as $A(p,s,e)=R_{\mathrm{pos}}^pR_sR_e$.
Each endpoint receives a structural address $A(i)\in G$, an operator playing the role that the rotation $R^t$ plays in RoPE.
Ordinary attention compares two endpoints through the relative operator between their addresses. We call this the journey operator (\Cref{ssec:journey}).
For standard RoPE this operator is the relative rotation $R^{j-i}$. LIS generates it from arbitrary structural indices rather than only integer offsets.
Thus the attention score receives role, instance, and position structure through the operator $P_{j\to i}=A(i)^{-1}A(j)$, rather than through integer coordinates assigned to a token layout whose relative offsets change with packing and context length.
LIS does not require turning text into structured data: ordinary text remains a special case of the same interface.
The point is that already-structured sources can enter Transformer attention with their roles, instances, and positions still explicit.

\noindent\textbf{Scale and alignment.}
By a repository we mean a large external store of facts, records, chunks, or cached keys.
At that scale, each entry's operator should be computed by a shared rule, not stored as a new parameter per entry.
In LIS, a repository entry can be treated as an instance $e$ whose operator is generated from content, for example $R_e=\varphi_\theta(\mathrm{content}_e)$, where $\varphi_\theta$ is a shared parametric map from entry content to an operator. Here content generates the operator used in the attention score, not a similarity key for selecting entries.
New entries are then addressed by the same learned rule, with no new parameters and no serialization.
When content is shared across text and structured sources, the content map $\tau$ from token occurrences to content identifiers, defined in \Cref{def:lis}, assumes that this alignment is available. Curated resources often provide it, while open-web alignment is an entity-linking problem outside this paper's scope.

\noindent\textbf{Contributions.}
This is a theory-oriented representation paper.
We make three main contributions:
\begin{enumerate}[nosep]
\item \textbf{A native structural address interface for scale} (\Cref{sec:scale,sec:framework}): text length, repositories, tables, code, KG facts, role markers, and structural identifiers all fail the same way: their coordinates are assigned to a serialized layout. LIS keeps endpoint content together with structural coordinates. With factored addresses, new data adds endpoints whose address comparisons reuse the operator factors learned in training.
\item \textbf{Operator properties give the unified address form} (\Cref{sec:framework}; \Cref{thm:closed-loop-role-algebra,prop:lis-special-cases,prop:coordinate-independence-commutation}): consistent journey operators give $P_{j\to i}=A(i)^{-1}A(j)$. Coordinate independence gives factorized addresses. Applying it to successive coordinate splits yields forms such as $A(p,s,e)=R_{\mathrm{pos}}^pR_sR_e$. RoPE, RoPE-2D, and HiRoPE are special cases. If the same independence is also required after reversing the product order, the corresponding factors must commute. \Cref{prop:coordinate-independence-commutation} states the exact condition for the role--instance case.
\item \textbf{Shallow diagnostics} (\Cref{sec:diagnostics}; Appendix~\ref{app:diagnostic-details}): controlled character-level experiments test whether the address mechanism can be implemented for ordinary text as well as KG/text conversion. A Tiny Shakespeare diagnostic treats text units as LIS instances using LIS-based attention, whose RoPE-style address for a character combines a within-line position factor with a content-derived line-instance factor. Validation perplexity is approximately constant from context 256 to 4096, showing that ordinary attention can compute and use this address. Synthetic KG/text diagnostics compare a role-marked serialized baseline against LIS variants with role-addressed KG arguments, including independent KG/text transfer, template-diverse KG$\to$text generation with content-derived fact addresses, and zero-shot stress tests on unseen names.
\end{enumerate}


\section{The Positional Bottleneck}\label{sec:scale}

Serialized positional coordinates must represent ordinary token order, text length, fact boundaries, KG roles, repository identity, table fields, and arbitrary packing order.
This causes two problems, whether the coordinates are 1D, multi-axis, or hierarchical.
Throughout, growth means any of the following: larger offsets in a flat coordinate, unseen values or unseen combinations in a multi-axis or hierarchical coordinate, or new entries requiring new identifiers.
Out of distribution, growth produces operators $P$ never seen in training: larger offsets in the flat case, unseen coordinate values or combinations in the multi-axis and hierarchical cases.
In distribution, $P$ has been seen in training, but its value depends on the serialization order rather than on a structural relation between the endpoints.
The common object in all of these cases is the operator $P$ inside each pairwise attention score,
\[
q^\top P k .
\]

\noindent\textbf{Out-of-distribution: length and support.}
Let a flat positional interface assign $A_{\mathrm{flat}}(t)=R_{\mathrm{pos}}^t$ to token position $t$, where $R_{\mathrm{pos}}$ is the one-token RoPE step operator.
If every training context has length at most $L_{\mathrm{train}}$, then every relative positional operator seen in training lies in
\[
\mathcal{P}_{\mathrm{train}}
=
\{R_{\mathrm{pos}}^\delta: |\delta|\leq L_{\mathrm{train}}-1\}.
\]
In a flat sequence of length $L$, the relative operator from position $j$ to position $i$ is
\[
A_{\mathrm{flat}}(i)^{-1}A_{\mathrm{flat}}(j)=R_{\mathrm{pos}}^{j-i}.
\]
Training therefore exposes only offsets $|\delta|\leq L_{\mathrm{train}}-1$.
At test length $L_{\mathrm{test}}>L_{\mathrm{train}}$, a relation can occur at an offset $|\delta|>L_{\mathrm{train}}-1$.
As contexts or repositories grow, these offsets can move arbitrarily far outside the trained region.
For standard RoPE's one-radian block, distinct integer offsets give distinct operators, so these are new operators outside $\mathcal{P}_{\mathrm{train}}$.
Serialization stays in-distribution only while all relevant coordinate offsets remain within the ranges seen in training.

\noindent\textbf{Out-of-distribution: context budget.}
A serialized KG, table, code object, or repository must fit into the finite token context before the Transformer can attend to it.
If we serialize more entries, their positional offsets can leave the ranges seen during training.
If we retrieve only a subset, the omitted entries are unavailable to the Transformer.

\noindent\textbf{Out-of-distribution: identifier budget.}
Learned identifier embeddings keep entries short only by assigning a parameter to each identifier.
A growing repository then needs new identifiers outside the trained table, collisions, or long serialized identifiers over a fixed vocabulary.
Long serialized identifiers are content strings, not structural addresses: they consume context, introduce more positions, and force the model to link the identifier to its entry through attention over the identifier's tokens.
The issue is the same as for long offsets: scale introduces offsets and identifiers the model never saw in training.

\noindent\textbf{In-distribution: packing dependence.}
When facts, text chunks, table rows, or code nodes are concatenated, their relative positions depend on arbitrary packing order.
Two entries can become near or far because a serializer placed them that way, not because the underlying structure says they have that relation.
The positional coordinate then encodes both the structure and an arbitrary layout choice.

\noindent\textbf{In-distribution: role markers, tags, and biases.}
Even when all serialized offsets lie inside the range of positions seen in training, common text-Transformer encodings such as role markers, additive tags, and scalar score biases put structural information in token content, value messages, embeddings, or scalar score terms rather than in the pairwise journey operator between endpoints.
For these encodings, the pairwise operator $P$ is still computed either from positions in the serialized layout or from entries in a learned label table, so the structural information added by markers, tags, and biases never reaches the operator that compares the two endpoints.
\Cref{thm:serialization-fragmentation} formalizes the positional part of this critique for role-marked structured facts: with variable entity lengths, the same semantic role pair appears as an instance-dependent family of serialized offsets, so no role--instance address factorization can be assigned to the serialized layout.

The key question is how $P$ is computed from endpoint coordinates, one integer or several, since that determines whether test-time operators stay in the trained family.

Flat position produces $P=R_{\mathrm{pos}}^\delta$ from a counter or serialized distance $\delta$. Growth can make test-time offsets, and hence test-time $P$, leave the trained family.
Packing creates the in-distribution variant of the same problem: the offset may be seen, but it is an artifact of serialization rather than a structural coordinate.
The same point extends to multi-coordinate rotary schemes such as RoPE-2D~\citep{heo2024ropevit} or HiRoPE~\citep{zhang2024hirope}: they split positional coordinates across axes or hierarchy levels, but $P$ is still generated from integer coordinates whose test-time values or combinations may be unseen during training.
Distance-biased schemes such as ALiBi handle large offsets by rule, but $P$ still depends on serialized distance, which packing order sets arbitrarily.
Learned identifier embeddings produce arbitrary operators $\Phi_j$ from labels $j$. A new label gives an unconstrained test-time operator.
Content-derived LIS addresses instead use a rule such as $R_e=\varphi_\theta(\mathrm{content}_e)$.

A retrieval embedding only selects entries. Attention still receives a serialized sequence with positional addresses.
LIS instead uses the content-derived operator inside $P$ itself.
Appendix~\ref{app:address-regime-score-facts} analyzes the same issue in terms of the attention score $q^\top Pk$. Addresses unconstrained by training data leave unseen scores undetermined, with learned identifier embeddings and discrete tags realizing the worst case. Content-derived operators change scores continuously with content. A serialization change that shifts a RoPE offset moves a fixed query--key score by an explicit amount.

\noindent\textbf{LIS alternative.}
LIS replaces coordinates whose values grow with the data by structural coordinates whose comparison operators are computed by fixed rules and reused across instances.
A sentence, paragraph, document section, table row, KG fact, and repository entry can all be treated as LIS instances: their participants differ, but each participant has content, a role, and an instance identity.
The formal construction is developed next.


\section{Related Work}\label{sec:related}

Hierarchical text and code models separate local and higher-level coordinates.
HIBERT encodes words inside each sentence and then sentences inside a document~\citep{zhang2019hibert}. PermGen uses hierarchical positional embeddings with a sentence-level global position and a token-level local position~\citep{yu2021sentencepermuted}. Hierarchical Attention Transformers process segments with segment-wise and cross-segment encoders~\citep{chalkidis2022hat}.
These approaches are close in spirit to resetting local positions within sentences or paragraphs and then adding a higher-level coordinate outside the local token coordinate.
HiRoPE makes the corresponding rotary move for code by replacing one flat code position with hierarchical integer coordinates split across RoPE dimensions~\citep{zhang2024hirope}.
In LIS notation, HiRoPE is a fixed-coordinate rotary address with commuting hierarchy factors.
It still computes addresses from integer coordinates, so growth can produce unseen values or combinations.
LIS keeps that operator view but allows a higher-level coordinate of the index to identify an instance or role, or allows an address factor to be generated from content, rather than limiting the index to fixed integer hierarchy levels.
The main distinction is how the index is used: LIS maps structural coordinates to factors of the endpoint address $A(i)$, which determines the journey operator in the attention score, rather than placing them only in token embeddings or segment-level processing.
Appendix~\ref{app:additional-related-work} gives additional comparisons to CoPE, randomized positional encodings, and structured-data, graph, and geometric models.


\section{Labeled Incidence Structures and Journey Operators}\label{sec:framework}

\subsection{Labeled Incidence Structures}\label{ssec:lis}

In a sequence, a token's location is one integer.
In structured data, such as a KG triple, a hyperedge, or a tree node, it has three parts: the relation instance the token belongs to, the role it plays within that instance, and its position within that role.

\begin{definition}[Labeled Incidence Structure]\label{def:lis}
A \emph{labeled incidence structure} (LIS) is a tuple $\mathcal{H} = (C, V, E, S, \tau, \mathcal{I}, \sigma)$ where:
\begin{itemize}[nosep]
\item $C$ is a set of \emph{content identifiers},
\item $V$ is a finite set of \emph{token occurrences},
\item $E$ is a finite set of \emph{instances} (hyperedges),
\item $S$ is a finite set of \emph{role labels}. A label may combine a semantic role with a within-role position.
\item $\tau : V \to C$ maps each token occurrence to its content identifier,
\item $\mathcal{I} \subseteq V \times E$ is the incidence relation,
\item $\sigma : \mathcal{I} \to S$ assigns roles to incidences.
\end{itemize}
For each instance $e$, the map $\sigma(\cdot, e)$ is injective on $\{v : (v,e) \in \mathcal{I}\}$: distinct occurrences in the same instance play distinct roles.
A semantic role that spans multiple tokens is represented by role labels that include within-role position, e.g. $(s,p)$.
\end{definition}

We use \emph{role} and \emph{slot} interchangeably: a slot is the formal role label in $S$.
An \emph{endpoint} is an incidence $(v,e)\in\mathcal I$ together with its content and structural index.
In the informal notation used in the introduction, it is the pair
\[
(x,i)=(\tau(v),(\sigma(v,e),e)),
\]
or, when within-role position is separated, $(x,i)=(\tau(v),(p,s,e))$.
Here and throughout, the \emph{structural index} $i$ is the tuple of coordinate values or labels attached to an endpoint.
The \emph{structural address} $A(i)$ is the invertible linear operator generated from that index.
Thus $i$ is the input to the address map $A$, not the address itself.
The distinction between token occurrences $V$ and content identifiers $C$ is crucial: the same entity (content) can appear multiple times in a document or knowledge base.
For example, ``Paris'' may occur in several sentences and KG triples simultaneously. Each occurrence is a distinct $v \in V$ with $\tau(v) = \texttt{Paris} \in C$.
Occurrences sharing a content identifier receive the same entity embedding but distinct role--instance assignments.

\begin{example}[Standard examples]\label{prop:canonical}
The structured data types considered here fit the same LIS representation: token occurrences, content identifiers, incidence membership, and role labels give the endpoint representation directly.
\begin{enumerate}[nosep]
\item[(a)] \textbf{Sequences}: one instance $e_{\mathrm{seq}}$, roles $S = \{1, \ldots, T\}$, $\sigma(v_t, e_{\mathrm{seq}}) = t$, with $\tau(v_t)$ = word type at position $t$.
\item[(b)] \textbf{KG triples}: each $(h,r,t)$ is an instance $e_{hrt}$ with roles $\{\texttt{HEAD}_r, \texttt{TAIL}_r\}$. The map $\tau$ sends head/tail occurrences to their entity identifiers.
\item[(c)] \textbf{Rooted trees}: each parent-child edge is an instance with roles $\{\texttt{PARENT}, \texttt{CHILD}\}$, covering document, HTML, and paragraph--sentence hierarchies.
\item[(d)] \textbf{$N$-dimensional grids}: one instance, roles = grid coordinates $(i_1, \ldots, i_N)$.
\item[(e)] \textbf{$n$-ary hypergraphs}: each hyperedge with $n$ participants is an instance with $n$ distinct role labels.
For example, the drug--gene--disease interaction (Aspirin, COX-2, Inflammation) is one instance $e$ with roles $\{\textsc{drug},\textsc{target},\textsc{disease}\}$.
\end{enumerate}
In (a), repeated words (e.g., ``the'' at positions 3 and 7) are distinct token occurrences $v_3 \neq v_7$ with $\tau(v_3) = \tau(v_7) = \texttt{the}$.
At this stage, positions are just role labels. The next subsection assigns operators to these labels.
\end{example}

\begin{example}[Multi-structure document]\label{ex:multi}
A biomedical document can contain both
\[
\begin{aligned}
e_1 &: \text{``Aspirin inhibits COX-2 to treat inflammation''},\\
e_2 &: \{\textsc{drug}:\text{Aspirin},\textsc{target}:\text{COX-2},\textsc{disease}:\text{Inflammation}\}.
\end{aligned}
\]
In the sentence instance $e_1$, roles are text positions such as $\texttt{POS}_1$, $\texttt{POS}_3$, and $\texttt{POS}_6$.
In the hyperedge instance $e_2$, roles are semantic roles such as \textsc{drug}, \textsc{target}, and \textsc{disease}.
These are two LIS instances that share content identifiers but use different roles and instance identifiers.
\Cref{tab:lis-example} in Appendix~\ref{app:lis-example} shows the construction.
\end{example}

\subsection{Journey Operators and Coordinate Factors}\label{ssec:journey}

\noindent\textbf{Properties of the journey operator.}
Having defined the incidence structure, we now ask how that structure should enter attention.
We consider RoPE-style relative attention scores~\citep{su2021roformer} of the form
\[
\operatorname{score}(i,j)=q_i^\top P_{j\to i}k_j .
\]
We refer to $P_{j\to i}$ as the \emph{journey operator}: it is the operator applied to the key from endpoint $j$ when the attention score at endpoint $i$ is computed.
The name reflects the consistency law below: applying the comparison through an intermediate endpoint must give the same operator as applying it directly.
It should satisfy three basic properties.
\begin{itemize}
\item \textbf{Linear score operator.}
$P_{j\to i}$ is a linear map applied to the key vector before the dot product.
For the address construction below, these structural maps are taken to be invertible operators in a matrix group $G\leq\GL(d)$.

\item \textbf{Journey consistency.}
This is the standard consistency condition for relative measurements: comparing directly should agree with comparing through an intermediate endpoint.
Comparing an endpoint with itself should give the identity.
Because $P_{j\to i}$ acts by matrix multiplication, the key vector $k_j$ can be compared with the query at endpoint $i$ either directly or through an intermediate endpoint $h$:
\[
k_j \mapsto P_{j\to h}k_j \mapsto P_{h\to i}P_{j\to h}k_j.
\]
Therefore the consistency identities are
\[
P_{i\to i}=I,\qquad P_{h\to i}P_{j\to h}=P_{j\to i}.
\]
Choosing a reference endpoint $0$ and setting $A(i)=P_{i\to 0}$, these identities force
\[
P_{j\to i}=A(i)^{-1}A(j) .
\]
Thus the structural comparison is a matrix product of endpoint addresses.

\item \textbf{Coordinate independence.}
Journey consistency gives endpoint addresses $A(i)$.
An endpoint may have several structural coordinates $(i_0,\ldots,i_n)$, such as position, role, instance, paragraph, document, or source.
For any split of the structural index into two coordinate blocks $i=(\alpha,\beta)$, write $Q_{\alpha,\beta}=A(\alpha,\beta)$.
Changing one block should have the same address effect regardless of the other block.
For two coordinate-block domains $\mathcal X$ and $\mathcal Y$, a complete two-block grid means that $Q_{\alpha,\beta}$ is defined for every $(\alpha,\beta)\in\mathcal X\times\mathcal Y$.
On such a grid, it is enough to require either of the following equivalent forms:
\[
\begin{aligned}
Q_{\alpha',\beta}Q_{\alpha,\beta}^{-1}
&\ \text{depends only on }\alpha,\alpha'\text{, not on }\beta,
\end{aligned}
\]
or
\[
\begin{aligned}
Q_{\alpha,\beta}^{-1}Q_{\alpha,\beta'}
&\ \text{depends only on }\beta,\beta'\text{, not on }\alpha.
\end{aligned}
\]
\end{itemize}

\noindent\textbf{From coordinate changes to factorization.}
For any two-block split, coordinate independence is exactly what allows an arbitrary address assignment to factor into one factor for each block.
Applying the two-block result to successive splits, for example
\[
(i_0,\ldots,i_n)= (i_0,\ldots,i_{n-1}) \mid i_n,
\]
and then splitting the left block again, gives an ordered product factorization for multiple coordinates whenever the same independence property holds for the required splits.
When a factorization exists, we adopt the left-to-right convention
\[
Q_{i_0,\ldots,i_n}=R_{i_0}\cdots R_{i_n},
\]
with factors written in the chosen coordinate order.
Any fixed order works, as long as it is used consistently.
The following two-block result is the standard consistency result for relative measurements. Related versions appear in group synchronization~\citep{singer2011angular}.

\begin{theorem}[Coordinate independence is equivalent to address factorization]\label{thm:closed-loop-role-algebra}
Let $\mathcal X$ and $\mathcal Y$ be two coordinate-block domains, assume $Q_{\alpha,\beta}$ is defined for every $(\alpha,\beta)\in\mathcal X\times\mathcal Y$ (a complete two-block grid), and let $Q_{\alpha,\beta}$ take values in a group $G\leq \GL(d)$.
An address assignment $Q_{\alpha,\beta}$ admits a fixed-order factorization
\[
Q_{\alpha,\beta}=R_\alpha R_\beta
\]
if and only if either of the following identities holds for all coordinate-block values $\alpha,\alpha'$ and $\beta,\beta'$. Equivalently, both hold:
\[
Q_{\alpha',\beta}Q_{\alpha,\beta}^{-1}
=Q_{\alpha',\beta'}Q_{\alpha,\beta'}^{-1},
\qquad
Q_{\alpha,\beta}^{-1}Q_{\alpha,\beta'}
=Q_{\alpha',\beta}^{-1}Q_{\alpha',\beta'}.
\]
When the factorization exists, it is unique up to a single change of reference frame: $R_\alpha \mapsto R_\alpha U$, $R_\beta \mapsto U^{-1} R_\beta$ for a fixed operator $U$.
\end{theorem}

The proof is a cancellation argument after fixing a reference coordinate. Appendix~\ref{app:flat-connection} gives the details and uniqueness calculation.

\begin{definition}[Role and instance operator assignment]\label{def:role-alg}
A \emph{role--instance operator assignment} over $\mathcal{H}$ is $(G, \{R_s\}_{s \in S}, \{R_e\}_{e \in E})$ with $G \leq \GL(d)$, role operators $R_s \in G$, and instance operators $R_e \in G$.
\end{definition}

The assignment records two things: role operators $R_s$ record \emph{what role} a token plays (head, tail, position 3, etc.), while instance operators $R_e$ record \emph{which instance} it belongs to (a specific triple, a specific sentence).
Together, the composite $R_s R_e$ gives each token's full structural address.
When a separate local position coordinate is used, we extend the same assignment by a fixed position generator $R_{\mathrm{pos}}$. Equivalently, position can be folded into the role label as in \Cref{def:lis}.

\begin{definition}[LIS Journey Operator]\label{def:journey-op}
Given a role--instance operator assignment, the LIS journey operator between role--instance pairs $(s,e)$ (source) and $(s',e')$ (target) is:
\begin{equation}\label{eq:journey}
P_{(s,e) \to (s',e')} = (R_{s'} R_{e'})^{-1} (R_s R_e) = R_{e'}^{-1} R_{s'}^{-1} R_s R_e.
\end{equation}
\end{definition}
By \Cref{thm:closed-loop-role-algebra}, this is exactly the journey form forced by journey consistency plus coordinate independence for the role--instance split after choosing the convention $A(s,e)=R_sR_e$.
Conversely, every journey rule of this form satisfies identity, composition, and coordinate independence.
The formula is the relative operator between the two endpoints' addresses.
In attention, it is applied to the key from the source endpoint before computing the score at the target endpoint.

The next result packages the standard rotary special cases.
Separating local position from semantic role refines a role label into a semantic role $s$ and a local position $p$.

\begin{proposition}[LIS contains RoPE, RoPE-2D, and HiRoPE]\label{prop:lis-special-cases}
Assume coordinate independence is imposed for the split $(p,s)\mid e$ and then for $p\mid s$, and assume the local position factors are generated by a single step operator, $R_p=R_{\mathrm{pos}}^p$.
Then a token endpoint with within-role position $p$, role $s$, and instance $e$ can use the ordered address
\[
A(p,s,e)=R_{\mathrm{pos}}^pR_sR_e .
\]
Plain text is the special case $A(t)=R_{\mathrm{pos}}^t$, with journey $P_{j\to i}=R_{\mathrm{pos}}^{j-i}$, recovering standard RoPE~\citep{su2021roformer}.
A 2D grid with $A(i,j)=R_x^iR_y^j$ and commuting axis operators recovers RoPE-2D~\citep{heo2024ropevit}.
HiRoPE is the same construction with a fixed number $h$ of integer hierarchy coordinates, for example $A(m_1,\ldots,m_h)=R_1^{m_1}\cdots R_h^{m_h}$ on commuting RoPE subspaces~\citep{zhang2024hirope}.
For a binary KG relation $r$, using roles $\texttt{HEAD}_r$ and $\texttt{TAIL}_r$ with no separate instance factor, or in the commuting rotation setting, gives the head-to-tail comparison $R_{\texttt{TAIL}_r}^{-1}R_{\texttt{HEAD}_r}$.
\end{proposition}

Thus one structural-address form covers text tokens, KG endpoints, and hyperedge endpoints.
As with RoPE, these fixed-coordinate rotary cases inherit the length-extrapolation problem of Section~\ref{sec:scale}, generalized to unseen coordinate values and combinations.
For a KG relation, the displayed head-to-tail comparison is the role-to-role comparison used by the attention score.

\begin{remark}[RotatE as value transport]\label{rem:rotate-value-transport}
If the operator is also applied to the value vectors, setting $R_r:=R_{\texttt{TAIL}_r}^{-1}R_{\texttt{HEAD}_r}$ gives the RotatE relation rotation $v_{\texttt{TAIL}}\approx R_r v_{\texttt{HEAD}}$~\citep{sun2019rotate}.
That extension is beyond the scope of this paper. The results here use journey operators in the attention comparison $q^\top Pk$.
\end{remark}

\begin{proposition}[Instance-independent role comparisons require commutation]\label{thm:factored-comembership}\label{prop:coordinate-independence-commutation}
Let $A(s,e)=R_sR_e$.
Write $Q_{s,e}=A(s,e)$.
The same-instance comparison in the opposite order is
\[
Q_{t,e}^{-1}Q_{s,e}=R_e^{-1}(R_t^{-1}R_s)R_e
\]
and is independent of $e$ exactly when $R_t^{-1}R_s$ commutes with every instance difference $R_{e'}R_e^{-1}$.
If this reversed-order independence is required for every role pair, then after normalizing one reference role and one reference instance to the identity, every normalized role operator commutes with every normalized instance operator.
\end{proposition}
The forward-order role change $Q_{t,e}Q_{s,e}^{-1}=R_tR_s^{-1}$ is instance-independent by cancellation.
The attention journey from \Cref{def:journey-op} is the reversed-order product. It may depend on the instance, and that dependence is how instance information can enter the score.
The proposition shows that asking this reversed product order to ignore instance as well imposes a commutation condition. Appendix~\ref{app:flat-connection} gives the proof.
Requiring both orders to be instance-independent forces commutation. If the order of role and instance factors is meant to carry information, the role and instance operators should not commute.


\section{Diagnostics with Shallow Transformers}\label{sec:diagnostics}

The diagnostics demonstrate three implementation claims: role/slot addresses can replace role-marker tokens in the attention score; text units, KG endpoints, and $n$-ary facts can be jointly trained as instances of the same LIS representation; and content-derived addresses can be computed and used inside ordinary Transformer attention.

We use diagnostics deliberately: these are small controlled tests of whether the address mechanism can be implemented and used inside ordinary attention.
They are not scale benchmarks. Larger experiments with multiple seeds, greater depth, and matched parameter counts are future work.
The plain-text diagnostic asks whether LIS is feasible for ordinary text, not only for KGs or structured records.
The LIS-based attention variant treats each newline-delimited line as a LIS instance.
A character endpoint has content, a local within-line position, and a line identity.
The model processes one character sequence causally, with RoPE-style query/key rotations generated from the line's local position and identity rather than only from global position.
Training loss and validation perplexity are computed over all tokens.

The KG/text diagnostics test the same principle for roles and fact identity.
The serialized baseline uses role-marker tokens and standard RoPE, while the LIS variants use relation-role indices to generate role-address factors and, in one diagnostic, add a content-derived fact-address factor.
These tests predict entities in a specified direction rather than copy strings.
Each example is still a single token sequence, but the LIS variants generate the address $A(i)$ from the endpoint's role and instance rather than from the serialized position alone.
Because KG facts and text templates name roles differently, the model must learn the cross-format role mapping before generating the target entity character by character.
Full setup, hit@1 scores where reported, parameter counts, fact-type tables, and metric definitions are in Appendix~\ref{app:diagnostic-details}.

\begin{table}[ht]
\centering
\caption{Representative diagnostics. The plain-text row reports all-token validation PPL from train context 256 to test context 4096. KG/text rows report entity-level h@5 (hit@5) over target names. Full ablations, h@1 scores where reported, and PPL results are in Appendix~\ref{app:diagnostic-details}. The main tiers are memorization, transfer, and generalization.}
\label{tab:main-strong-diagnostics}
\footnotesize
\setlength{\tabcolsep}{3pt}
\resizebox{\columnwidth}{!}{%
\begin{tabular}{@{}llcc l@{}}
\toprule
\textbf{Diagnostic} & \textbf{Condition} & \textbf{Flat/Serialized} & \textbf{LIS} & \textbf{What is tested} \\
\midrule
Tiny Shakespeare & PPL, 256$\to$4096 chars & 4.70$\to$13.36 & 4.81$\to$4.86 & text units as LIS instances \\
Independent KG/text & text eval, main tiers & .248--.269 & .401--.445 & KG-only facts to text queries \\
Independent KG/text & KG eval, main tiers & .280--.296 & .475--.510 & text-only facts to KG queries \\
Template-diverse KG$\to$text & single-fact test, chain2/3 & .126--.359 & .923--.972 & unseen names with one KG fact \\
Template-diverse KG$\to$text & same-entities test, chain2/3 & .030--.411 & .835--.952 & relation choice among distractor facts \\
Template-diverse KG$\to$text & same-relation test, chain2/3 & .000--.109 & .038--.182 & hardest entity-binding stress test \\
\bottomrule
\end{tabular}
}
\end{table}

\Cref{tab:main-strong-diagnostics} first checks the plain-text implementation: LIS-based attention treats each line as an instance and computes its address from line content inside the attention operator.
The remaining rows summarize the KG/text diagnostics where the serialized baseline reaches nonzero accuracy. In the 1,000-name zero-shot settings of Appendix~\ref{app:diagnostic-details} the serialized baseline scores zero, so those settings test only whether the LIS mechanism works, not the size of the gap.


\section{Discussion and Conclusion}\label{sec:conclusion}

LIS keeps endpoint content and structural coordinates explicit, so adding text, repository entries, table rows, or KG facts adds endpoints addressed by the same learned rule, instead of producing larger offsets, unseen coordinate values, or unseen coordinate combinations.
The formal results derive the journey form, characterize coordinate factorization, and show how serialization ties role addresses to token length.

The diagnostics demonstrate that role/slot addresses, joint LIS training of text units and KG endpoints, and content-derived addresses can all be implemented inside ordinary Transformer attention.
These implementations require no extra heads, blocks, or attention pass.
The measured training overhead comes from computing content-derived addresses.
When structural addresses are supplied directly, this cost is absent.
In the plain-text diagnostic, the line address is computed once from input embeddings and reused in RoPE-style query/key rotations.
This adds $O(T)$ work, while full attention is $O(T^2)$.
Appendix~\ref{app:timing-benchmark} reports training-step timing measurements.
In those runs, relative overhead decreases with model size.
Larger experiments with multiple seeds, greater depth, and matched parameter counts are the next step.

\bibliographystyle{apalike}
\bibliography{references}

@article{su2021roformer,
  title={RoFormer: Enhanced Transformer with Rotary Position Embedding},
  author={Su, Jianlin and Lu, Yu and Pan, Shengfeng and Murtadha, Ahmed and Wen, Bo and Liu, Yunfeng},
  journal={arXiv preprint arXiv:2104.09864},
  year={2021}
}

@inproceedings{press2022alibi,
  title={Train Short, Test Long: Attention with Linear Biases Enables Input Length Extrapolation},
  author={Press, Ofir and Smith, Noah A. and Lewis, Mike},
  booktitle={International Conference on Learning Representations},
  year={2022}
}

@article{chen2023positioninterpolation,
  title={Extending Context Window of Large Language Models via Positional Interpolation},
  author={Chen, Shouyuan and Wong, Sherman and Chen, Liangjian and Tian, Yuandong},
  journal={arXiv preprint arXiv:2306.15595},
  year={2023}
}

@article{peng2023yarn,
  title={{YaRN}: Efficient Context Window Extension of Large Language Models},
  author={Peng, Bowen and Quesnelle, Jeffrey and Fan, Honglu and Shippole, Enrico},
  journal={arXiv preprint arXiv:2309.00071},
  year={2023}
}

@inproceedings{ruoss2023randomized,
  title={Randomized Positional Encodings Boost Length Generalization of Transformers},
  author={Ruoss, Anian and Del{\'e}tang, Gr{\'e}goire and Genewein, Tim and Grau-Moya, Jordi and Csord{\'a}s, R{\'o}bert and Bennani, Mehdi and Legg, Shane and Veness, Joel},
  booktitle={Proceedings of the 61st Annual Meeting of the Association for Computational Linguistics (Volume 2: Short Papers)},
  pages={1889--1903},
  year={2023},
  address={Toronto, Canada},
  publisher={Association for Computational Linguistics},
  doi={10.18653/v1/2023.acl-short.161},
  url={https://aclanthology.org/2023.acl-short.161/}
}

@article{golovneva2024cope,
  title={Contextual Position Encoding: Learning to Count What's Important},
  author={Golovneva, Olga and Wang, Tianlu and Weston, Jason and Sukhbaatar, Sainbayar},
  journal={arXiv preprint arXiv:2405.18719},
  year={2024},
  url={https://arxiv.org/abs/2405.18719}
}

@inproceedings{heo2024ropevit,
  title={Rotary Position Embedding for Vision Transformer},
  author={Heo, Byeongho and Park, Song and Han, Dongyoon and Yun, Sangdoo},
  booktitle={European Conference on Computer Vision},
  year={2024},
  doi={10.48550/arXiv.2403.13298}
}

@inproceedings{zhang2024hirope,
  title={{HiRoPE}: Length Extrapolation for Code Models Using Hierarchical Position},
  author={Zhang, Kechi and Li, Ge and Zhang, Huangzhao and Jin, Zhi},
  booktitle={Proceedings of the 62nd Annual Meeting of the Association for Computational Linguistics (Volume 1: Long Papers)},
  pages={13615--13627},
  year={2024},
  address={Bangkok, Thailand},
  publisher={Association for Computational Linguistics},
  doi={10.18653/v1/2024.acl-long.735},
  url={https://aclanthology.org/2024.acl-long.735/}
}

@inproceedings{zhang2019hibert,
  title={{HIBERT}: Document Level Pre-training of Hierarchical Bidirectional Transformers for Document Summarization},
  author={Zhang, Xingxing and Wei, Furu and Zhou, Ming},
  booktitle={Proceedings of the 57th Annual Meeting of the Association for Computational Linguistics},
  pages={5059--5069},
  year={2019},
  doi={10.18653/v1/P19-1499}
}

@inproceedings{yu2021sentencepermuted,
  title={Sentence-Permuted Paragraph Generation},
  author={Yu, Wenhao and Zhu, Chenguang and Zhao, Tong and Guo, Zhichun and Jiang, Meng},
  booktitle={Proceedings of the 2021 Conference on Empirical Methods in Natural Language Processing},
  pages={5051--5062},
  year={2021},
  doi={10.18653/v1/2021.emnlp-main.412}
}

@article{chalkidis2022hat,
  title={An Exploration of Hierarchical Attention Transformers for Efficient Long Document Classification},
  author={Chalkidis, Ilias and Dai, Xiang and Fergadiotis, Manos and Malakasiotis, Prodromos and Elliott, Desmond},
  journal={arXiv preprint arXiv:2210.05529},
  year={2022}
}

@inproceedings{sun2019rotate,
  title={{RotatE}: Knowledge Graph Embedding by Relational Rotation in Complex Space},
  author={Sun, Zhiqing and Deng, Zhi-Hong and Nie, Jian-Yun and Tang, Jian},
  booktitle={International Conference on Learning Representations},
  year={2019}
}

@article{yao2019kgbert,
  title={{KG-BERT}: {BERT} for Knowledge Graph Completion},
  author={Yao, Liang and Mao, Chengsheng and Luo, Yuan},
  journal={arXiv preprint arXiv:1909.03193},
  year={2019}
}

@inproceedings{sun2020colake,
  title={{CoLAKE}: Contextualized Language and Knowledge Embedding},
  author={Sun, Tianxiang and Shao, Yunfan and Qiu, Xipeng and Guo, Qipeng and Hu, Yaru and Huang, Xuanjing and Zhang, Zheng},
  booktitle={Proceedings of the 28th International Conference on Computational Linguistics},
  pages={3660--3670},
  year={2020}
}

@inproceedings{feng2019hypergraph,
  title={Hypergraph Neural Networks},
  author={Feng, Yifan and You, Haoxuan and Zhang, Zizhao and Ji, Rongrong and Gao, Yue},
  booktitle={AAAI Conference on Artificial Intelligence},
  volume={33},
  year={2019}
}

@inproceedings{Liu2024HyperGT,
  title={Hypergraph Transformer for Semi-Supervised Classification},
  author={Liu, Zexi and Tang, Bohan and Ye, Ziyuan and Dong, Xiaowen and Chen, Siheng and Wang, Yanfeng},
  booktitle={ICASSP 2024},
  pages={7515--7519},
  year={2024}
}

@inproceedings{Hu2020HGT,
  title={Heterogeneous Graph Transformer},
  author={Hu, Ziniu and Dong, Yuxiao and Wang, Kuansan and Sun, Yizhou},
  booktitle={Proceedings of The Web Conference 2020},
  pages={2704--2710},
  publisher={ACM},
  year={2020},
  doi={10.1145/3366423.3380027}
}

@article{singer2011angular,
  title={Angular Synchronization by Eigenvectors and Semidefinite Programming},
  author={Singer, Amit},
  journal={Applied and Computational Harmonic Analysis},
  volume={30},
  number={1},
  pages={20--36},
  year={2011}
}

@inproceedings{cohen2019gauge,
  title={Gauge Equivariant Convolutional Networks and the Icosahedral {CNN}},
  author={Cohen, Taco and Weiler, Maurice and Kicanaoglu, Berkay and Welling, Max},
  booktitle={Proceedings of the 36th International Conference on Machine Learning},
  series={Proceedings of Machine Learning Research},
  volume={97},
  pages={1321--1330},
  year={2019}
}

@inproceedings{dehaan2021gauge,
  title={Gauge Equivariant Mesh {CNN}s: Anisotropic Convolutions on Geometric Graphs},
  author={de Haan, Pim and Weiler, Maurice and Cohen, Taco and Welling, Max},
  booktitle={International Conference on Learning Representations},
  year={2021}
}

@inproceedings{he2021gauge,
  title={Gauge Equivariant Transformer},
  author={He, Lingshen and Dong, Yiming and Wang, Yisen and Tao, Dacheng and Lin, Zhouchen},
  booktitle={Advances in Neural Information Processing Systems},
  volume={34},
  year={2021}
}

@inproceedings{schlichtkrull2018rgcn,
  title={Modeling Relational Data with Graph Convolutional Networks},
  author={Schlichtkrull, Michael and Kipf, Thomas N. and Bloem, Peter and van den Berg, Rianne and Titov, Ivan and Welling, Max},
  booktitle={European Semantic Web Conference (ESWC)},
  year={2018}
}

@article{wen2016beyondbinary,
  title={On the Representation and Embedding of Knowledge Bases Beyond Binary Relations},
  author={Wen, Jianfeng and Li, Jianxin and Mao, Yongyi and Chen, Shini and Zhang, Richong},
  journal={arXiv preprint arXiv:1604.08642},
  year={2016}
}

@inproceedings{guan2019nalp,
  title={Link Prediction on N-ary Relational Data},
  author={Guan, Saiping and Jin, Xiaolong and Wang, Yuanzhuo and Cheng, Xueqi},
  booktitle={The World Wide Web Conference},
  pages={583--593},
  publisher={ACM},
  year={2019},
  doi={10.1145/3308558.3313414}
}

@inproceedings{rosso2020hinge,
  title={Beyond Triplets: Hyper-Relational Knowledge Graph Embedding for Link Prediction},
  author={Rosso, Paolo and Yang, Dingqi and Cudr{\'e}-Mauroux, Philippe},
  booktitle={Proceedings of The Web Conference 2020},
  pages={1885--1896},
  publisher={ACM},
  year={2020},
  doi={10.1145/3366423.3380257}
}

@inproceedings{fatemi2020hype,
  title={Knowledge Hypergraphs: Prediction Beyond Binary Relations},
  author={Fatemi, Bahare and Taslakian, Perouz and Vazquez, David and Poole, David},
  booktitle={Proceedings of the Twenty-Ninth International Joint Conference on Artificial Intelligence},
  pages={2191--2197},
  year={2020},
  doi={10.24963/ijcai.2020/303}
}

@article{fatemi2023real,
  title={Knowledge Hypergraph Embedding Meets Relational Algebra},
  author={Fatemi, Bahare and Taslakian, Perouz and Vazquez, David and Poole, David},
  journal={Journal of Machine Learning Research},
  volume={24},
  number={105},
  pages={1--34},
  year={2023}
}

@inproceedings{ying2021graphormer,
  title={Do Transformers Really Perform Bad for Graph Representation?},
  author={Ying, Chengxuan and Cai, Tianle and Luo, Shengjie and Zheng, Shuxin and Ke, Guolin and He, Di and Shen, Yanming and Liu, Tie-Yan},
  booktitle={Advances in Neural Information Processing Systems},
  volume={34},
  year={2021}
}

@inproceedings{kim2022tokengt,
  title={Pure Transformers are Powerful Graph Learners},
  author={Kim, Jinwoo and Nguyen, Tien Dat and Min, Seonwoo and Cho, Sungjun and Lee, Moontae and Lee, Honglak and Hong, Seunghoon},
  booktitle={Advances in Neural Information Processing Systems},
  volume={35},
  year={2022},
  doi={10.52202/068431-1060}
}

@article{dwivedi2023benchmarking,
  title={Benchmarking Graph Neural Networks},
  author={Dwivedi, Vijay Prakash and Joshi, Chaitanya K. and Luu, Anh Tuan and Laurent, Thomas and Bengio, Yoshua and Bresson, Xavier},
  journal={Journal of Machine Learning Research},
  volume={24},
  number={43},
  pages={1--48},
  year={2023}
}

@article{fagin1977multivalued,
  title={Multivalued Dependencies and a New Normal Form for Relational Databases},
  author={Fagin, Ronald},
  journal={ACM Transactions on Database Systems},
  volume={2},
  number={3},
  pages={262--278},
  year={1977},
  doi={10.1145/320557.320571}
}

@book{abiteboul1995foundations,
  title={Foundations of Databases},
  author={Abiteboul, Serge and Hull, Richard and Vianu, Victor},
  publisher={Addison-Wesley},
  year={1995}
}

@misc{noy2006defining,
  title={Defining {N}-ary Relations on the Semantic Web},
  author={Noy, Natasha and Rector, Alan},
  howpublished={{W3C} Working Group Note},
  year={2006},
  url={https://www.w3.org/TR/swbp-n-aryRelations/}
}

\newpage
\appendix

The appendix follows the main paper's order: Appendix~\ref{app:serialization-fragmentation} gives the serialization result used in Section~\ref{sec:scale}; Appendix~\ref{app:address-regime-score-facts} collects results stated in terms of the attention score, namely an indistinguishability bound for unconstrained addresses, with learned identifier embeddings and discrete tags as examples; a continuity modulus for content-derived addresses; and the explicit offset modulus for rotary schemes. Appendix~\ref{app:additive-embeddings} records the additive-embedding comparison; Appendices~\ref{app:lis-example}--\ref{app:diagnostic-details} give the LIS example, proofs, and diagnostics; Appendix~\ref{app:additional-related-work} collects additional related work.

\section{Serialization Ties Role Addresses to Token Length}\label{app:serialization-fragmentation}

\begin{theorem}[Serialized offsets for a role pair depend on entity lengths]\label{thm:serialization-fragmentation}
Consider a fixed-order lossless text serialization of $m$-ary instances:
\[
\langle r_1\rangle\,x_1(e)\,\langle r_2\rangle\,x_2(e)\cdots
\langle r_m\rangle\,x_m(e),
\]
where role markers have fixed token lengths and the token length $\ell_u(e)$ of each role value $x_u(e)$ may vary by instance.
Let $p_s(e)$ be the flattened sequence position of the first entity token in role $s$.
Assign the serialized endpoint the RoPE-style positional address $Q_{s,e}=R_{\mathrm{pos}}^{p_s(e)}$, where $R_{\mathrm{pos}}$ is the one-token RoPE step operator.
For any $s<t$,
\[
p_t(e)-p_s(e)=\kappa_{s,t}+\sum_{u=s}^{t-1}\ell_u(e),
\]
where $\kappa_{s,t}$ depends only on the serialization template.
Consequently, suppose two instances $e,e'$ differ in the intervening length sum
\[
\delta_{s,t}(e,e')
:=\sum_{u=s}^{t-1}\bigl(\ell_u(e')-\ell_u(e)\bigr).
\]
If $\delta_{s,t}(e,e')\neq0$ and $R_{\mathrm{pos}}^{\delta_{s,t}(e,e')}\neq I$ for this discrepancy, then the serialized positional addresses admit no factorization $Q_{s,e}=R_sR_e$ on the four endpoints $\{s,t\}\times\{e,e'\}$.
\end{theorem}

\begin{remark}[Standard RoPE satisfies the nonaliasing hypothesis]\label{rem:standard-rope-no-aliasing}
For block-diagonal RoPE generators, the hypothesis means that no nonzero realized discrepancy aliases to the identity simultaneously in all frequency blocks.
For standard RoPE frequencies, exact aliasing would require all frequency blocks to return to the identity at the same integer offset.
The highest-frequency block rotates by one radian per token, so no nonzero integer offset can make that block exactly identity. Hence the full RoPE operator is nonidentity for every nonzero integer discrepancy.
\end{remark}

\begin{proof}[Proof of \Cref{thm:serialization-fragmentation}]
Write the fixed serialization template as alternating role markers and role values.
Let $\mu_r$ be the fixed token length of role marker $\langle r\rangle$.
For an instance $e$, the position of the first entity token in role $s$ is
\[
p_s(e)=b(e)+\sum_{u<s}\bigl(\mu_u+\ell_u(e)\bigr)+\mu_s,
\]
where $b(e)$ is the start position of the serialized instance.
Therefore, for $s<t$,
\[
p_t(e)-p_s(e)
=
\sum_{u=s}^{t-1}\ell_u(e)+\sum_{u=s+1}^{t}\mu_u
=
\kappa_{s,t}+\sum_{u=s}^{t-1}\ell_u(e),
\]
with $\kappa_{s,t}:=\sum_{u=s+1}^{t}\mu_u$ independent of $e$.
If the intervening entity-length sums differ for two instances, then the flattened offsets for the same role pair differ.
RoPE-style relative addressing over the flattened sequence uses the relative-position operator $R_{\mathrm{pos}}^{p_t(e)-p_s(e)}$. Hence the operator attached to role pair $(s,t)$ varies with the instance's tokenization pattern, up to any aliasing of the chosen positional generator.

Now assign each serialized endpoint the address $Q_{s,e}=R_{\mathrm{pos}}^{p_s(e)}$ and suppose a factorization $Q_{s,e}=R_sR_e$ existed on the four endpoints $\{s,t\}\times\{e,e'\}$.
Direct cancellation would give the rectangle identity
\[
Q_{s,e'}Q_{t,e'}^{-1}Q_{t,e}Q_{s,e}^{-1}=I
\]
from \Cref{thm:closed-loop-role-algebra}.
But substituting the serialized positional addresses gives
\[
\begin{aligned}
Q_{s,e'}Q_{t,e'}^{-1}Q_{t,e}Q_{s,e}^{-1}
&=
R_{\mathrm{pos}}^{p_s(e')}R_{\mathrm{pos}}^{-p_t(e')}R_{\mathrm{pos}}^{p_t(e)}R_{\mathrm{pos}}^{-p_s(e)} \\
&=
R_{\mathrm{pos}}^{p_s(e')-p_t(e')+p_t(e)-p_s(e)}.
\end{aligned}
\]
Using the offset identity above,
\begin{align*}
p_s(e')-p_t(e')+p_t(e)-p_s(e)
&=-\bigl(p_t(e')-p_s(e')\bigr)\\
&\quad+\bigl(p_t(e)-p_s(e)\bigr)\\
&=-\sum_{u=s}^{t-1}\bigl(\ell_u(e')-\ell_u(e)\bigr).
\end{align*}
Thus
\[
Q_{s,e'}Q_{t,e'}^{-1}Q_{t,e}Q_{s,e}^{-1}=R_{\mathrm{pos}}^{-\delta_{s,t}(e,e')}.
\]
If $R_{\mathrm{pos}}^{\delta_{s,t}(e,e')}\neq I$, this contradicts the rectangle identity, so the serialized positional address cannot be factored as $R_sR_e$ on that role--instance rectangle.
For commuting RoPE rotations, such a factorization would imply an additive role-plus-instance decomposition in every frequency block: $\omega_b p_s(e)\equiv a_{b,s}+b_{b,e}\pmod{2\pi}$ for all blocks $b$.
	The nonidentity rectangle product is exactly the obstruction to these per-block decompositions.
\end{proof}

\begin{corollary}[Fixed-width fields are the positional repair]\label{cor:fixed-width-fields}
Assume every instance contains roles $1,\ldots,m$.
Assume $R_{\mathrm{pos}}^\delta\neq I$ for every nonzero realized discrepancy.
Then the serialized positional address has coordinate independence on every two-role, two-instance rectangle if and only if each non-final role length $\ell_s(e)$, $s<m$, is instance-independent.
The final role length does not enter any within-instance role-pair offset in this fixed-order serialization.
Thus for data with genuinely variable lengths in roles that precede another role, removing this positional-address obstruction by fixed fields requires padding those roles to their maximum realized lengths. If the maximum is $L$, the serialized instance uses $\Theta(mL)$ tokens rather than $\sum_u\ell_u(e)$ actual content tokens.
\end{corollary}

\begin{proof}
Coordinate independence forces $\delta_{s,t}(e,e')=0$ on every rectangle. Taking adjacent roles $t=s+1$ gives $\ell_s(e')=\ell_s(e)$ for each $s<m$.
The converse is immediate from the definition of $\delta_{s,t}$, and the padding claim follows by using fixed-width fields for the roles whose lengths enter later offsets.
\end{proof}

\begin{remark}[Address-level scope]\label{rem:serialization-mask-scope}
Serialization is lossless as storage but can make the positional address for a semantic role pair depend on length configurations.
In the commuting-rotation or no-instance-factor settings, the LIS role address keeps the role-pair journey shared across instances. In the general noncommuting setting, instance dependence is allowed and can carry structural information.
The argument concerns only the address assignment $Q_{s,e}=R_{\mathrm{pos}}^{p_s(e)}$, not the attention mask or prediction objective, so the same address-level obstruction applies to bidirectional MLM encoders and causal decoders over flattened $n$-ary text.
\end{remark}


\section{How Addressing Schemes Affect Attention Scores}\label{app:address-regime-score-facts}

This section states results directly about the attention score $q^\top Pk$.
The results make one point: addresses unconstrained by training data leave unseen scores undetermined, while content-derived addresses change scores continuously.

\begin{proposition}[Unconstrained addresses leave unseen scores undetermined]\label{prop:regime-necessity}
Let a score input be $z=(q,k,\iota)$, where $\iota$ is the input from which the addressing scheme produces the operator $P(\iota)$ used in the score $q^\top P(\iota)k$.
Let training data contain score inputs drawn from a set $\mathcal Z_{\mathrm{train}}$, and let $z_\star=(q_\star,k_\star,\iota_\star)$ be a test score input, possibly revealed to the learner.
Let $\mathcal F$ be a candidate class of score rules $f(z)$ for the direct attention score.
Suppose there exist $f_0,f_1\in\mathcal F$ such that
\[
f_0(z)=f_1(z)\quad\text{for every }z\in\mathcal Z_{\mathrm{train}},
\qquad
|f_0(z_\star)-f_1(z_\star)|=\Delta .
\]
Then every possibly randomized learner whose input is the training data and $z_\star$ has worst-case expected absolute error at least $\Delta/2$ over $\{f_0,f_1\}$.
For squared loss, the corresponding worst-case expected error is at least $\Delta^2/4$.
\end{proposition}

\begin{proof}
The training observations and the revealed test input $z_\star$ are identical under $f_0$ and $f_1$, so any randomized learner induces the same prediction distribution for $\hat s$ in both cases.
Let $a=f_0(z_\star)$ and $b=f_1(z_\star)$.
For absolute loss,
\[
\mathbb E|\hat s-a|+\mathbb E|\hat s-b|
\geq |a-b|=\Delta
\]
by the triangle inequality, so one of the two expected absolute errors is at least $\Delta/2$.
For squared loss,
\[
(\hat s-a)^2+(\hat s-b)^2
=2\left(\hat s-\frac{a+b}{2}\right)^2+\frac{(a-b)^2}{2}
\geq \frac{\Delta^2}{2}
\]
pointwise.
Taking expectations, one of the two expected squared errors is at least $\Delta^2/4$.
\end{proof}

\begin{example}[Learned identifier embeddings and discrete tags realize the worst case]\label{ex:unseen-id-floor}
Let learned identifier embeddings assign an orthogonal operator $\Phi_j$ to each identifier $j$.
For an identifier $j^\star$ unseen in training, fix a unit vector $q$ and set $k=q$.
Take $\Phi^{(1)}_{j^\star}=I$ and $\Phi^{(2)}_{j^\star}=I-2qq^\top$, and let the two assignments agree on all trained identifiers and all contents.
The test scores are $+1$ and $-1$.
The same construction applies to discrete tag addresses.
With $B=R_t^{-1}R_s$, unit $q$, and $k=B^{-1}q$, assign arbitrarily close contents $z,u$ the tags $U_z=I$ and $U_u=B^{-1}(I-2qq^\top)B$.
Against a fixed target tag $U_w=I$, the two scores are $1$ and $-1$.
Content places no constraint here: the bound holds even if the learner observes the content of $j^\star$ or of the tagged items.
\end{example}

Coverage of the trained range, a shared addressing rule tying $\iota_\star$ to trained inputs, or a Lipschitz bound are three ways to exclude the indistinguishable pairs in \Cref{prop:regime-necessity}.
Content-derived addresses exclude such pairs locally: for a fixed query and key, nearby address inputs can only induce nearby scores.
For deterministic counters such as RoPE, the operator at an unseen offset is fixed by the rule.
The issue is not unseen-identifier ambiguity. The learned query/key projections were fit against the rotations seen in training, and \Cref{rem:rope-score-sensitivity} gives the exact modulus by which the direct score moves when the offset changes, including when that change carries the operator outside the positional range seen during training.

\noindent\textbf{Repository entries as content-addressed instances.}
For a repository entry with content descriptor $\mathrm{content}_e$, LIS treats the entry as an instance $e$ and computes its instance operator by a shared rule, e.g. $R_e=\varphi_\theta(\mathrm{content}_e)$.
An endpoint in that entry has address $A(s,e)=R_sR_e$, and the key used by attention can be pre-rotated and cached as $\hat k=A(s,e)k$ when the content and model parameters are fixed.
This is not a claim that every query must attend densely over every repository entry.
An implementation can still use a selection index, cache, or gating mechanism to choose a manageable set of addressed keys.
The difference is that the selected entry keeps its content-derived address. It is not re-serialized into a string whose relations are encoded as token offsets.

\noindent\textbf{Content-derived addresses.}
For repository entries or cached facts, the address can be computed from entry content, schema, or an enclosing hierarchy and stored with the key when those quantities are fixed.
For lines, sentences, or facts already present in a prefix, the same address form can be computed from the tokens in that unit.
Both cases generate the operator used in $q^\top Pk$.

\begin{proposition}[Content-derived instance operators change scores continuously]\label{thm:content-transfer}
Let $(\mathcal Z,\rho)$ be a metric space of instance contents, and let $\varphi:\mathcal Z\to \mathrm O(d)$ be $L$-Lipschitz in operator norm.
Fix orthogonal role operators $\{R_s\}_{s\in S}$ and assign an instance with content $z$ the address $A(s,z)=R_s\varphi(z)$.
For $\|q\|,\|k\|\leq1$, define
\[
F_{s,t}(z,z';q,k)=q^\top A(t,z')^{-1}A(s,z)k .
\]
Then
\[
\left|F_{s,t}(z,z';q,k)-F_{s,t}(u,u';q,k)\right|
\leq L\bigl(\rho(z,u)+\rho(z',u')\bigr).
\]
Consequently, if every test instance content is within $\eta$ of a training instance content, then the corresponding fixed-query/key structural scores are within $2L\eta$.
\end{proposition}

\begin{proof}
Write $B_{s,t}=R_t^{-1}R_s$ and
\[
A(t,z')^{-1}A(s,z)=\varphi(z')^{-1}B_{s,t}\varphi(z).
\]
For two content pairs $(z,z')$ and $(u,u')$,
\begin{align*}
&\left\|\varphi(z')^{-1}B_{s,t}\varphi(z)-\varphi(u')^{-1}B_{s,t}\varphi(u)\right\|_{\mathrm{op}}\\
&\qquad\leq
\|\varphi(z)-\varphi(u)\|_{\mathrm{op}}
+\|\varphi(z')-\varphi(u')\|_{\mathrm{op}}
\leq L\bigl(\rho(z,u)+\rho(z',u')\bigr),
\end{align*}
using orthogonality and
$\|\varphi(z')^{-1}-\varphi(u')^{-1}\|_{\mathrm{op}}
=\|\varphi(z')-\varphi(u')\|_{\mathrm{op}}$.
Multiplication by unit vectors $q,k$ cannot increase the operator-norm difference.
\end{proof}

\begin{remark}[Offset changes move scores by an explicit amount]\label{rem:rope-score-sensitivity}
For block-diagonal RoPE with $2\times2$ rotation blocks $R(\omega_b\delta)$, orthogonality gives
\[
\sup_{\|q\|=\|k\|=1}
\left|q^\top\bigl(R_{\mathrm{pos}}^\delta-R_{\mathrm{pos}}^{\delta'}\bigr)k\right|
=
\left\|R_{\mathrm{pos}}^{\Delta}-I\right\|_{\mathrm{op}}
=
\max_b 2\left|\sin\!\left(\frac{\omega_b\Delta}{2}\right)\right|.
\]
Here $\Delta=\delta-\delta'$.
For standard RoPE with a one-radian-per-token block, this is at least $2|\sin(\Delta/2)|$: a serialization change that shifts an offset by one token can move a fixed query--key score by order one.
The same blockwise calculation covers RoPE-2D and HiRoPE with the scalar offset replaced by the coordinate-vector phase change.
\end{remark}

\section{Additive Structural Embeddings and Score Biases}\label{app:additive-embeddings}

This section does not introduce a new architecture.
It compares additive structural embeddings and score biases with putting structure directly into the pairwise journey matrix.
Additive structural embeddings, endpoint tags, and score biases can add role-dependent affine terms around a shared content--content bilinear score.
A role-pair-dependent matrix between the two content vectors contains all these additive forms as special cases. LIS operators are one structured way to produce such matrices.

\begin{proposition}[Additive embeddings and biases are restricted score forms]\label{prop:additive-embeddings}
Fix content dimension $d\geq1$ and let $\Lambda$ contain at least two role-pair labels.
Let $\mathcal A$ be the class of additive one-layer structural score families
\[
\ell_{\mathcal A}(\lambda;c_1,c_2)
=c_1^\top M c_2+u_\lambda^\top c_1+v_\lambda^\top c_2+\beta_\lambda,
\]
for $\lambda\in\Lambda$, with one shared content matrix $M$ and role-pair-dependent linear and constant terms.
This class contains the direct score produced by additive structural embeddings $x_r=c+a_r$ with linear query/key projections, as well as additive role-pair score biases.
Let $\mathcal L$ be the class of score families
\[
\ell_{\mathcal L}(\lambda;c_1,c_2)=\tilde c_1^\top P_\lambda\tilde c_2,\qquad
\tilde c=(c,1),
\]
where the matrix between the two augmented content vectors may depend on the role pair $\lambda$.
Then $\mathcal A\subsetneq\mathcal L$.
The strict separating example can be chosen to be LIS-realizable: two role pairs whose content blocks are $I$ and $-I$ with zero affine blocks.
\end{proposition}

\begin{proof}
First, additive structural embeddings with linear projections have the displayed affine form.
Write $x_r=c+a_r$ and let $M=W_Q^\top W_K$.
For source role $r_1$ and target role $r_2$,
\[
(c_1+a_{r_1})^\top M(c_2+a_{r_2})
=c_1^\top M c_2
+(Ma_{r_2})^\top c_1
+(M^\top a_{r_1})^\top c_2
+a_{r_1}^\top M a_{r_2}.
\]
Thus additive structural embeddings contribute role-pair-dependent linear terms and constants around one shared content matrix.
An additive score bias simply changes the constant term $\beta_\lambda$.

For any score in $\mathcal A$, set $\tilde c=(c,1)$ and
\[
B_\lambda=
\begin{bmatrix}
M & u_\lambda\\
v_\lambda^\top & \beta_\lambda
\end{bmatrix}.
\]
Then
\[
\tilde c_1^\top B_\lambda\tilde c_2
=c_1^\top M c_2+u_\lambda^\top c_1+v_\lambda^\top c_2+\beta_\lambda
=\ell_{\mathcal A}(\lambda;c_1,c_2),
\]
so every additive embedding or bias score in this class belongs to $\mathcal L$.
This containment is for the general score form. An arbitrary family of matrices need not satisfy LIS journey consistency.
If one insists that this matrix come from an invertible journey operator, embed $B_\lambda$ as the visible block of the larger matrix
\[
\widehat B_\lambda=
\begin{bmatrix}
B_\lambda&I\\
I&0
\end{bmatrix},
\qquad
\widehat c=(\tilde c,0).
\]
The matrix $\widehat B_\lambda$ is invertible for every $B_\lambda$: if
$\widehat B_\lambda(x,y)=(B_\lambda x+y,x)=0$, then $x=0$ and hence $y=0$.
Moreover $\widehat c_1^\top\widehat B_\lambda\widehat c_2=\tilde c_1^\top B_\lambda\tilde c_2$.

The containment is strict.
Let $M_0\neq0$ be any $d\times d$ matrix and choose two role-pair labels $\lambda_+$ and $\lambda_-$.
The class $\mathcal L$ represents
\[
\ell_+(c_1,c_2)=c_1^\top M_0c_2,\qquad
\ell_-(c_1,c_2)=-c_1^\top M_0c_2
\]
by taking $P_{\lambda_+}=M_0$ and $P_{\lambda_-}=-M_0$ on the content block.
For the LIS-valid orthogonal example, take $M_0=I$.
Then $P_{\lambda_+}=I$ and $P_{\lambda_-}=-I$ are obtained from role operators $R_a=I$, $R_b=I$, and $R_c=-I$ via same-instance comparisons $R_a^{-1}R_b$ and $R_a^{-1}R_c$.

No additive interface can represent both $\ell_+$ and $\ell_-$ with one shared matrix $M$ between the entity contents.
Representing $\ell_+$ for all $c_1,c_2$ and setting $c_2=0$ forces $u_{\lambda_+}=0$ and $\beta_{\lambda_+}=0$. Setting $c_1=0$ forces $v_{\lambda_+}=0$. Hence $M=M_0$.
Representing $\ell_-$ by the same argument forces $M=-M_0$, a contradiction since $M_0\neq0$.
\end{proof}


\section{LIS Encoding Example}\label{app:lis-example}

Table~\ref{tab:lis-example} shows only the participating occurrences from the sentence and hyperedge. Omitted non-entity text tokens account for the gaps in the occurrence numbering.

\begin{table}[h]
\centering\small
\caption{LIS encoding of a text sentence (instance $e_1$) and a 3-ary drug--gene--disease hyperedge (instance $e_2$). Shared content identifiers link the two structures. Text positions and hyperedge roles are both role labels.}
\label{tab:lis-example}
\resizebox{\columnwidth}{!}{%
\begin{tabular}{@{}lllll@{}}
\toprule
\textbf{Occurrence $v$} & $\boldsymbol{\tau(v)}$ & \textbf{Instance $e$} & \textbf{Slot $s$} & \textbf{Address} \\
\midrule
$v_1$ (``Aspirin'', text) & Aspirin & $e_1$ (sentence) & $\texttt{POS}_1$ & $R_{\texttt{POS}_1} R_{e_1}$ \\
$v_3$ (``COX-2'', text) & COX-2 & $e_1$ (sentence) & $\texttt{POS}_3$ & $R_{\texttt{POS}_3} R_{e_1}$ \\
$v_6$ (``inflammation'', text) & Inflammation & $e_1$ (sentence) & $\texttt{POS}_6$ & $R_{\texttt{POS}_6} R_{e_1}$ \\
\midrule
$v_7$ (``Aspirin'', drug) & Aspirin & $e_2$ (hyperedge) & \textsc{drug} & $R_{\textsc{drug}} R_{e_2}$ \\
$v_8$ (``COX-2'', target) & COX-2 & $e_2$ (hyperedge) & \textsc{target} & $R_{\textsc{target}} R_{e_2}$ \\
$v_9$ (``Inflammation'', disease) & Inflammation & $e_2$ (hyperedge) & \textsc{disease} & $R_{\textsc{disease}} R_{e_2}$ \\
\bottomrule
\end{tabular}
}
\end{table}

\section{Coordinate Independence and Address Factorization}\label{app:flat-connection}

\begin{proof}[Proof of \Cref{thm:closed-loop-role-algebra}]
First suppose $Q_{\alpha,\beta}=R_\alpha R_\beta$.
For a change in the first block while the second block is fixed at $\beta$,
\[
Q_{\alpha',\beta}Q_{\alpha,\beta}^{-1}
=
R_{\alpha'}R_\beta(R_\alpha R_\beta)^{-1}
=
R_{\alpha'}R_\alpha^{-1},
\]
which is independent of $\beta$.
For a change in the second block while the first block is fixed at $\alpha$,
\[
Q_{\alpha,\beta}^{-1}Q_{\alpha,\beta'}
=
(R_\alpha R_\beta)^{-1}R_\alpha R_{\beta'}
=
R_\beta^{-1}R_{\beta'},
\]
which is independent of $\alpha$.
Thus every factorized address assignment has coordinate independence.

Conversely, assume the first coordinate-independence identity and fix a reference first-block value $\alpha_0$.
By the first coordinate identity, the left first-block ratio
\[
C_\alpha:=Q_{\alpha,\beta}Q_{\alpha_0,\beta}^{-1}
\]
is independent of the second block $\beta$.
Therefore
\[
Q_{\alpha,\beta}
=
C_\alpha Q_{\alpha_0,\beta}.
\]
Set $R_\alpha:=C_\alpha$ and $R_\beta:=Q_{\alpha_0,\beta}$.
Then $Q_{\alpha,\beta}=R_\alpha R_\beta$, so the address factors as claimed.
The forward direction already showed that any such factorization implies both coordinate-independence identities, so the second displayed identity follows automatically.

If instead only the second coordinate-independence identity is assumed, fix a reference second-block value $\beta_0$ and define
\[
D_\beta:=Q_{\alpha,\beta_0}^{-1}Q_{\alpha,\beta}.
\]
The second coordinate identity makes $D_\beta$ independent of $\alpha$, and therefore
\[
Q_{\alpha,\beta}=Q_{\alpha,\beta_0}D_\beta.
\]
Setting $R_\alpha:=Q_{\alpha,\beta_0}$ and $R_\beta:=D_\beta$ gives the same fixed-order factorization.

It remains to prove uniqueness. Suppose $Q_{\alpha,\beta}=R_\alpha R_\beta=\widetilde R_\alpha\widetilde R_\beta$ are two factorizations.
For every first-block and second-block value, $R_\alpha^{-1}\widetilde R_\alpha = R_\beta\widetilde R_\beta^{-1}$.
The left side depends only on $\alpha$ and the right side only on $\beta$. Since the
identity holds for all pairs, this common value is the same for all first-block
and second-block values. Call it $U$.
Then $\widetilde R_\alpha=R_\alpha U$ and $\widetilde R_\beta=U^{-1}R_\beta$.
\end{proof}

\noindent\textbf{Rectangle form.}
For later use, the first identity in \Cref{thm:closed-loop-role-algebra} is equivalently the triviality of the four-step product
\[
L_Q(\alpha,\alpha';\beta,\beta')
:=
Q_{\alpha,\beta'}Q_{\alpha',\beta'}^{-1}Q_{\alpha',\beta}Q_{\alpha,\beta}^{-1}
=I .
\]
This is the rectangle product used in Appendix~\ref{app:serialization-fragmentation}.

\noindent\textbf{Multi-coordinate consequence.}
Let $i=(i_0,\ldots,i_n)$ range over a complete product of coordinate values, and assume coordinate independence for every split of these coordinates into two blocks.
Apply \Cref{thm:closed-loop-role-algebra} first to the split
\[
(i_0,\ldots,i_{n-1}) \mid i_n .
\]
This gives $A(i_0,\ldots,i_n)=B(i_0,\ldots,i_{n-1})R_{i_n}$.
Fix a reference value of $i_n$.
The intermediate address assignment $B$ differs from the restricted addresses $A(i_0,\ldots,i_{n-1},i_n^0)$ only by a fixed right factor, so it inherits the first coordinate-independence identity for every split of $(i_0,\ldots,i_{n-1})$.
On a complete grid, \Cref{thm:closed-loop-role-algebra} then gives the equivalent second identity as well.
Repeating the argument gives
\[
A(i_0,\ldots,i_n)=R_{i_0}R_{i_1}\cdots R_{i_n}
\]
in the chosen coordinate order.

\noindent\textbf{Order convention and commutation.}
The ordered product above uses a fixed coordinate order and does not require the factors to commute.
If the same endpoint must have the same address under arbitrary permutations of the coordinate order, then adjacent swaps force the corresponding coordinate-factor families to commute.
Thus commutation is an additional order-invariance requirement, not part of ordered address factorization itself.

\begin{proof}[Proof of \Cref{prop:coordinate-independence-commutation}]
Write $Q_{s,e}=A(s,e)$ and $B_{s,t}=R_t^{-1}R_s$.
For a fixed instance $e$,
\[
Q_{t,e}^{-1}Q_{s,e}
=(R_tR_e)^{-1}(R_sR_e)
=R_e^{-1}B_{s,t}R_e.
\]
This reversed product is independent of $e$ exactly when, for every pair $e,e'$,
\[
R_e^{-1}B_{s,t}R_e=R_{e'}^{-1}B_{s,t}R_{e'}.
\]
Multiplying on the left by $R_{e'}$ and on the right by $R_e^{-1}$ gives
\[
(R_{e'}R_e^{-1})B_{s,t}=B_{s,t}(R_{e'}R_e^{-1}),
\]
which is precisely commutation with $R_{e'}R_e^{-1}$.
The converse follows by reversing the same algebra.

Now fix a reference role $s_0$ and reference instance $e_0$.
Normalize the factors so that $R_{s_0}=I$ and $R_{e_0}=I$.
Taking the role pair $(s,s_0)$ gives the normalized role factor $\bar R_s=R_{s_0}^{-1}R_s$, and taking the instance pair $(e,e_0)$ gives the normalized instance factor $\bar R_e=R_eR_{e_0}^{-1}$.
The relative commutation condition therefore becomes ordinary commutation:
\[
\bar R_s\bar R_e=\bar R_e\bar R_s
\qquad\text{for all roles }s\text{ and instances }e.
\]
\end{proof}

\noindent\textbf{LIS journey formula.}
Substituting the factorized address $A(s,e)=R_sR_e$ from \Cref{thm:closed-loop-role-algebra} into the consistent journey form $P_{a\to b}=A(b)^{-1}A(a)$ gives \Cref{eq:journey}. The converse is the direct cancellation calculation already used above.

\section{Diagnostic Details}\label{app:diagnostic-details}

\subsection{Plain-Text Context Growth with Content-Derived Line Addresses}\label{app:tiny-shakespeare-lisformer}

This diagnostic tests whether LIS can be implemented for ordinary text.
It treats each completed line as an instance and asks whether a longer context can be handled as more instances rather than as larger values in one positional coordinate system.
All four models are character-level causal Transformers trained on Tiny Shakespeare with a 65-token vocabulary, four layers, four attention heads, embedding dimension 128, pre-norm residual blocks, GELU activations, and scaled dot-product attention.
Parameter counts range from about 810K to 819K.
Training uses 1M characters with a 90/10 train/validation split, 256-character chunks, 5,000 iterations, batch size 64, AdamW with learning rate $3\times 10^{-4}$, cosine schedule, weight decay 0.1, and gradient clipping at 1.0.
Cross-entropy loss is computed at every token.
After training, validation perplexity is measured over all tokens in evaluation windows of 256, 512, 1024, 2048, and 4096 characters.

\noindent\textbf{Models.}
The continuous baseline uses standard RoPE positions that increment from $0$ to $T-1$ across the full context.
The ALiBi baseline adds a linear distance bias to attention scores, using slopes $2^{-2}$, $2^{-4}$, $2^{-6}$, and $2^{-8}$ across its four heads.
The random-line-address model and the content-addressed model use the same causal implementation and differ only in how they produce line addresses.
Both reset local RoPE positions at each newline.
The random-line-address model samples an independent angle vector $\xi_e\in\R^{d/2}$ for each completed line on every forward pass, with entries drawn i.i.d. from $\mathcal N(0,1)$.
This gives lines different addresses without using line content.
LISformer treats each newline-delimited line as an LIS instance.
Each character endpoint has the form $(x,(p,e))$, where $x$ is the character content, $p$ is the within-line position, and $e$ is the line instance.
The RoPE-style query/key address is generated from $(p,e)$ rather than from global positional coordinates.
Thus LISformer changes the rotary operator angles used in the attention score, not the token embeddings.
LISformer uses the same within-line reset and adds a content-derived line-instance address for completed lines.
For each completed line, token embeddings are first rotated by their within-line RoPE angles, scatter-added into a line sum, mean-pooled by line length, and then projected through LayerNorm and a linear map to an angle vector in $\R^{d/2}$.
That angle offset is added to the RoPE angles of every token in the completed line.
The final incomplete line receives zero line offset, so $R_e=I$ in both addressed models.
Every prediction uses only preceding text.
Within one line, the shared line offset cancels in relative attention angles, so intra-line attention sees ordinary local RoPE.

\begin{table}[ht]
\centering
\caption{All-token validation perplexity on Tiny Shakespeare after 5,000 training iterations on 256-character chunks. Content-derived line addresses remain approximately constant as evaluation context grows to 16 times the training length. Each row reports one run.}
\label{tab:tiny-shakespeare-lisformer}
\footnotesize
\setlength{\tabcolsep}{6pt}
\begin{tabular}{@{}lccccc@{}}
\toprule
\textbf{Model} & \textbf{256 chars} & \textbf{512 chars} & \textbf{1024 chars} & \textbf{2048 chars} & \textbf{4096 chars} \\
\midrule
Continuous RoPE & 4.70 & 5.99 & 8.23 & 10.71 & 13.36 \\
ALiBi & 4.96 & 4.94 & 4.96 & 4.93 & 4.93 \\
Random line address & 4.89 & 4.79 & 4.90 & 4.87 & 4.91 \\
LISformer (content address) & 4.81 & 4.74 & 4.85 & 4.83 & 4.86 \\
\bottomrule
\end{tabular}
\end{table}

\noindent\textbf{Interpretation.}
Continuous RoPE encodes line identity only through the growing global position and degrades as context exceeds the training length.
ALiBi, random line addresses, and content-derived line addresses produce approximately constant curves in these runs.
LISformer keeps within-line positions bounded while computing completed-line addresses from content.
Its all-token perplexity changes from 4.81 to 4.86 and is lower than the random-address model's at every evaluated context length in this run.
The table shows that this LIS text address can be implemented inside ordinary attention and has the intended behavior: adding context adds addressed text instances rather than increasing the local positional range.
A flat curve does not by itself show how much distant information the model uses.

\subsection{Timing Benchmark}\label{app:timing-benchmark}

We measured training-step overhead for the plain-text continuous-RoPE baseline (B) and LISformer (J) on an NVIDIA Tesla T4 GPU with 15.8~GB memory and compute capability 7.5, using PyTorch 2.6.0, CUDA 12.4, and Python 3.12.
These timings are not scale benchmarks. They measure the implementation overhead of computing content-derived line addresses in the same RoPE-style Transformer code path.
They are separate address-computation benchmarks, not training-step measurements of the corrected causal runs in \Cref{tab:tiny-shakespeare-lisformer}.

Each model was initialized with \texttt{torch.manual\_seed(42)} and dropout set to 0.0.
For each configuration, one fixed training batch was drawn with \texttt{get\_batch}.
We ran one warmup training step, then timed ten training steps, each consisting of forward pass, backward pass, optimizer update, and gradient zeroing, with \texttt{torch.cuda.synchronize()} around the timed region.
Between models, the model and optimizer were deleted and the CUDA cache was cleared.
Context length was fixed at 1024 for all configurations. Batch size was reduced at larger scales to fit GPU memory.
The optimizer was AdamW with learning rate $3\times 10^{-4}$, weight decay 0.1, cosine annealing, and gradient clipping at 1.0.

The models use multi-head rotary attention with full query, key, value, and output projections, implemented through \texttt{F.scaled\_dot\_product\_attention} with causal masking.
RoPE is applied per head to queries and keys after reshaping to $(B,H,T,D)$, using the standard frequency schedule $1/10000^{2i/d}$ for $i=0,\ldots,d/2-1$.
Each block is pre-norm residual attention followed by a GELU MLP with $4\times$ expansion.
The token embedding is followed by transformer blocks, final LayerNorm, and an output linear layer, with no weight tying and no learned position embedding.
LISformer adds vectorized line detection, scatter-add mean pooling over completed lines, and a line-angle projection \texttt{LayerNorm(d)} followed by \texttt{Linear(d,d/2)}.

\begin{table}[ht]
\centering
\caption{Training-step timing benchmark for the continuous-RoPE baseline (B) and LISformer (J). Each entry reports wall-clock milliseconds per training step, averaged over 10 timed steps after one warmup step, at context length 1024 on an NVIDIA Tesla T4.}
\label{tab:timing-benchmark}
\footnotesize
\setlength{\tabcolsep}{5pt}
\begin{tabular}{@{}lrrrrrr@{}}
\toprule
\textbf{Scale} & \textbf{$d$} & \textbf{Layers} & \textbf{Heads} & \textbf{Batch} & \textbf{B step (ms)} & \textbf{J overhead} \\
\midrule
0.8M & 128 & 4 & 4 & 8 & 42 & +10.3\% \\
85M & 768 & 12 & 12 & 4 & 702 & +4.4\% \\
302M & 1024 & 24 & 16 & 2 & 1215 & +3.5\% \\
680M & 1536 & 24 & 16 & 1 & 1428 & +2.0\% \\
\bottomrule
\end{tabular}
\end{table}

The LISformer additions are $O(Td)$ per layer for fixed hidden size: vectorized line detection, one within-line rotation before pooling, scatter-add mean pooling, and a line-angle projection.
Transformer training includes attention and MLP work of order $O(T^2d+Td^2)$, so in these timing runs the relative overhead decreases as model size grows.

\subsection{Role-Addressing KG/Text Diagnostics}\label{app:kgtext-diagnostics}

\noindent\textbf{Shallow synthetic diagnostic.}
We include a small diagnostic to check whether the LIS role-addressing mechanism helps in a shallow model.
The question is whether direct role/instance addresses improve entity prediction when the baseline represents the same roles with serialized relation tokens and positions.
We construct a 3-ary synthetic configuration: 3,000 synthetic lowercase names of length 2--8 arranged in familial, workplace, and military chains, rendered both as text templates and as KG-form records.
The baseline is a serialized relation-token model. LIS uses the same entity strings and prediction targets, but represents roles through direct role addresses rather than through relation tokens.
Both models are three-layer causal transformers trained with the same text+KG objective and evaluated only on entity prediction with the target role or relation already specified.
\Cref{tab:lis-synthetic} reports entity-level hit@5 and per-character perplexity (PPL) on the same target entity characters.
This diagnostic checks whether the LIS role-addressing difference remains visible when the serialized baseline learns nontrivially.

Entity strings are tokenized character by character, while role and relation markers are single special tokens.
This diagnostic uses the full 85-token vocabulary for the softmax, not just the lowercase entity-character subset. The zero-shot stress tests below use a larger 175-token vocabulary because they include a larger role-token and special-token inventory.
Evaluation scores entity prediction with the target role or relation already specified.
The top three tiers are memorization for facts seen in training, transfer for held-out derived facts whose base facts were seen, and generalization for held-out derived facts whose derivation must generalize. The label \texttt{kg\_excl} marks names seen only in KG-form training facts, and \texttt{txt\_excl} marks names seen only in text-form training facts.
The suffixes \texttt{mem} and \texttt{gen} distinguish memorization-style and generalization-style queries within the exclusive-name splits.
For hit@1/hit@5, a target entity string counts as correct only if every character of the entity is among the model's top one/top five candidates at its prediction step. Scores are averaged over target entities.
Perplexity (PPL) is computed on the same target entity characters as a per-character geometric mean.

\begin{table}[ht]
\centering
\caption{Synthetic LIS diagnostic: serialized relation-token baseline versus LIS, using embedding dimension 500, three layers, and 3,000 synthetic names.}
\label{tab:lis-synthetic}
\footnotesize
\setlength{\tabcolsep}{6pt}
\textbf{Text evaluation}

\begin{tabular}{@{}lcccc@{}}
\toprule
\textbf{Tier} & \textbf{Serialized h@5} & \textbf{LIS h@5} & \textbf{Serialized PPL} & \textbf{LIS PPL} \\
\midrule
memorization & .259 & .445 & 3.59 & 3.10 \\
transfer & .269 & .440 & 3.47 & 3.08 \\
generalization & .248 & .401 & 3.83 & 3.48 \\
kg\_excl\_mem & .184 & .329 & 4.68 & 3.97 \\
kg\_excl\_gen & .189 & .323 & 4.92 & 4.18 \\
txt\_excl\_mem & .119 & .202 & 5.10 & 4.62 \\
txt\_excl\_gen & .107 & .163 & 5.23 & 4.98 \\
\bottomrule
\end{tabular}

\vspace{0.5em}
\textbf{KG evaluation}

\begin{tabular}{@{}lcccc@{}}
\toprule
\textbf{Tier} & \textbf{Serialized h@5} & \textbf{LIS h@5} & \textbf{Serialized PPL} & \textbf{LIS PPL} \\
\midrule
memorization & .288 & .506 & 3.19 & 2.73 \\
transfer & .296 & .510 & 3.13 & 2.73 \\
generalization & .280 & .475 & 3.27 & 2.87 \\
kg\_excl\_mem & .231 & .418 & 3.44 & 3.01 \\
kg\_excl\_gen & .230 & .413 & 3.63 & 3.14 \\
txt\_excl\_mem & .111 & .190 & 5.94 & 5.13 \\
txt\_excl\_gen & .104 & .168 & 5.82 & 5.46 \\
\bottomrule
\end{tabular}
\end{table}

The LIS model has higher h@5 and lower PPL than the serialized baseline in every tier for both text and KG evaluation.
The differences are largest on memorization, transfer, generalization, and KG-exclusive tiers, while text-exclusive tiers show the same direction.
On the main memorization, transfer, and generalization tiers, serialized h@5 is .248--.296 while LIS h@5 is .401--.510.
This gives a role-addressing comparison in a setting where the serialized baseline is substantially nonzero.

We make the causal setup fair to the serialized baseline.
Both models see all role permutations used by the causal prediction task, so every participant can appear after every other participant in some training string. Role markers move with the entities they mark in each serialized string.
The architecture, optimizer, and entity-scoring protocol belong to the same family as the zero-shot diagnostics, except that this diagnostic uses embedding dimension 500 and 3,000 synthetic names.
The data format is different: here KG-form and text-form facts are trained as independent serialized sentences rather than as paired KG$\mid$text lines of the form used in the zero-shot diagnostics below, which is why the main text describes this table as independent KG/text transfer.

\noindent\textbf{Mixed text--KG zero-shot variant.}
We also test whether structural KG encodings transfer to completely unseen entity names.
We compare two character-level transformer architectures on serialized KG/text pairs.
These diagnostics test whether role addresses preserve the KG/text correspondence when entity names are unseen.
All reported PPL values are character-level PPL computed only over target entity-name characters, not over the full serialized sequence.
The task is not plain string copying.
The model conditions on one KG/text format and generates another, so unseen-name success requires using the learned role correspondence to generate the target entity character by character.
The two zero-shot tables below are stress tests for the address mechanism.
The serialized baseline uses standard RoPE positional encoding. KG facts are linearized into text with explicit role tokens, e.g. \texttt{<son> adam <father> brian}, and both KG and text are processed as ordinary token sequences with position-based angles.
The LIS/slot-angle model augments RoPE with learned per-relation angular offsets for KG entity slots.
Each of the 57 relation types is represented over four canonical chain slots, e.g. a four-position window such as son, father, grandfather, great-grandfather in the family domain.
A 2-ary or 3-ary fact uses the two or three slots that participate in that relation.
This is the experimental parameterization of the role operator $R_s$: each canonical slot stores one learned phase per rotary frequency block, giving a block-diagonal rotation factor that enters the attention address before token mixing.
Its KG side is still a serialized entity string, but with space-delimited entities and no role tokens. The model identifies each entity's role through its angular offset rather than through explicit role tokens.
These zero-shot LIS variants use the no-separate-instance-factor case of \Cref{prop:lis-special-cases}: they test role addresses, while content-derived fact addresses are introduced in the template-diverse diagnostic below.
For both causal models, KG entity-role units are randomly permuted during training. This is not a no-role baseline because the serialized model receives explicit role-marker content in every permutation.
Natural-language tokens receive no learned relation or slot angle. Relation information on that side comes from the text/template words.

\noindent\textbf{Architecture and training.}
Both models are single-head character-level transformers with pre-norm residual blocks: LayerNorm, rotary attention, LayerNorm, and a ReLU feed-forward network with $4\times$ expansion.
The attention layer applies RoPE rotation to queries and keys before scaled dot-product attention with softmax.
The zero-shot variants use a 175-token vocabulary: lowercase letters, space, period, apostrophe, digits, role tokens for the serialized baseline, and special tokens.
Models are trained on paired serialized strings KG $\mid$ text, where the KG part encodes a relational fact and the text part expresses the same fact in natural language.
Training uses causal language-modeling loss over the full sequence, AdamW with learning rate $5\times 10^{-4}$, gradient clipping at 1.0, batch size 32, dropout 0.2, and 100K training updates.
Both zero-shot variants in \Cref{tab:mixed-zero-shot,tab:bidirectional-zero-shot} use embedding dimension 60, three layers, block size 192, 1000 training names, and 100 test names.
The LIS model adds learned slot-angle parameters for relation slots. The remaining architecture and training protocol are the same.
With the stated architecture, the zero-shot serialized baseline has about 153K trainable parameters and the LIS variant about 160K. The difference is the 6,840-parameter slot-angle table, $57$ relation types $\times$ $4$ canonical slots $\times$ $30$ rotary frequency blocks.
For the 3,000-name diagnostic in \Cref{tab:lis-synthetic}, embedding dimension 500 gives about 9.106M parameters for the serialized baseline and 9.163M for LIS. The slot-angle table has 57,000 parameters.
The reported diagnostics are single-seed runs. Seed sweeps are left to the larger empirical study noted in \Cref{sec:conclusion}.

\noindent\textbf{Data and evaluation.}
The diagnostics use disjoint random strings of 2--8 lowercase characters.
Names are arranged into chains in three domains: family, work, and military.
Each chain position gives facts at three granularities: base 2-ary facts such as \texttt{son\_of} and \texttt{father\_of}, base 3-ary facts such as \texttt{grandson\_of} with an intermediate entity, and derived 2-ary facts such as \texttt{great\_grandson\_of} spanning four chain positions.
In total there are 57 distinct relation types across the three domains, with arities 2 and 3.
At evaluation, the conditioned serialized side contains test names unseen during training, and the model autoregressively predicts the other serialized side.
Each entity in the fact is predicted separately, given the others in the text template.
For LIS, evaluation uses the same slot-angle mechanism as training, with the text entity length growing by one at each autoregressive step.
The zero-shot PPL values are character-level PPL computed only on the target entity-name characters.
Because the test names are disjoint random strings, a trained model can give the correct unseen name very low probability while confidently predicting training-name characters. PPL can therefore be worse than uniform even though the evaluation contains no template tokens.

\noindent\textbf{Experiment 1: pure KG$\to$text training.}
\Cref{tab:mixed-zero-shot} reports the pure KG$\to$text variant: training and evaluation condition on the KG encoding and predict the corresponding text.
The LIS model has much higher unseen-name scores than the serialized baseline in this stress test.
Aggregated over fact types, LIS reaches 1.000 hit@5 and 1.08 PPL on train names, and .398 hit@5 and 6.56 PPL on unseen test names.
The serialized baseline reaches .126 hit@5 and 5.36 PPL on train names, but zero hit@5 and 545 PPL on test names.
This small-width setting is intentionally a harsher stress test than \Cref{tab:lis-synthetic}: the serialized baseline underfits even the training names, so we report it as a stress test of the addressing mechanism, not as a tuned comparison.

\begin{table}[ht]
\centering
\caption{Pure KG$\to$text zero-shot diagnostic. Models use embedding dimension 60, three layers, block size 192, 100K updates, 1000 training names, and 100 unseen test names. Text and KG facts are placed in the same line during training. Evaluation conditions on KG and predicts text.}
\label{tab:mixed-zero-shot}
\scriptsize
\setlength{\tabcolsep}{4pt}
\textbf{By fact type}

\begin{tabular}{@{}llcccccc@{}}
\toprule
\textbf{Fact type} & \textbf{Name set} & \textbf{Serialized h@1} & \textbf{Serialized h@5} & \textbf{Serialized PPL} & \textbf{LIS h@1} & \textbf{LIS h@5} & \textbf{LIS PPL} \\
\midrule
base\_2ary & train & .002 & .130 & 5.38 & .951 & 1.000 & 1.08 \\
base\_2ary & test & .000 & .000 & 577 & .182 & .397 & 6.44 \\
base\_3ary & train & .001 & .107 & 5.39 & .970 & 1.000 & 1.06 \\
base\_3ary & test & .000 & .000 & 515 & .185 & .413 & 6.44 \\
derived & train & .003 & .140 & 5.29 & .912 & 1.000 & 1.10 \\
derived & test & .000 & .000 & 544 & .165 & .385 & 6.80 \\
\bottomrule
\end{tabular}

\vspace{0.5em}
\textbf{Aggregated}

\begin{tabular}{@{}lcccccc@{}}
\toprule
\textbf{Name set} & \textbf{Serialized h@1} & \textbf{Serialized h@5} & \textbf{Serialized PPL} & \textbf{LIS h@1} & \textbf{LIS h@5} & \textbf{LIS PPL} \\
\midrule
train & .002 & .126 & 5.36 & .944 & 1.000 & 1.08 \\
test & .000 & .000 & 545 & .177 & .398 & 6.56 \\
\bottomrule
\end{tabular}
\end{table}

\noindent\textbf{Experiment 2: bidirectional training.}
We next train with 50/50 KG-first and text-first sequences, evaluating in both directions in \Cref{tab:bidirectional-zero-shot}.
The LIS model again has higher h@5 and lower PPL than the serialized baseline across all conditions.
In the KG$\to$text direction, LIS achieves .999 hit@5 versus .089 for the serialized baseline on train names, and .301 hit@5 versus .000 on unseen test names.
In the text$\to$KG direction, where the model predicts native KG entity names given text, LIS achieves .995 hit@5 versus .088 for the serialized baseline on train names, and .360 hit@5 versus .000 on unseen test names.
The h@1 and PPL columns show the same pattern.

\begin{table}[ht]
\centering
\caption{Bidirectional mixed text--KG zero-shot diagnostic. Models use embedding dimension 60, three layers, block size 192, 100K updates, 1000 training names, and 100 unseen test names. Models are trained and evaluated in both KG$\to$text and text$\to$KG directions.}
\label{tab:bidirectional-zero-shot}
\scriptsize
\setlength{\tabcolsep}{4pt}
\textbf{KG$\to$text}

\begin{tabular}{@{}llcccccc@{}}
\toprule
\textbf{Fact type} & \textbf{Name set} & \textbf{Serialized h@1} & \textbf{Serialized h@5} & \textbf{Serialized PPL} & \textbf{LIS h@1} & \textbf{LIS h@5} & \textbf{LIS PPL} \\
\midrule
base\_2ary & train & .000 & .097 & 5.50 & .927 & 1.000 & 1.14 \\
base\_2ary & test & .000 & .000 & 516 & .148 & .310 & 10.37 \\
base\_3ary & train & .000 & .090 & 5.46 & .928 & .999 & 1.14 \\
base\_3ary & test & .000 & .000 & 517 & .143 & .307 & 10.34 \\
derived & train & .000 & .080 & 5.33 & .919 & .999 & 1.16 \\
derived & test & .000 & .000 & 506 & .115 & .287 & 12.97 \\
\bottomrule
\end{tabular}

\vspace{0.5em}
\textbf{Text$\to$KG}

\begin{tabular}{@{}llcccccc@{}}
\toprule
\textbf{Fact type} & \textbf{Name set} & \textbf{Serialized h@1} & \textbf{Serialized h@5} & \textbf{Serialized PPL} & \textbf{LIS h@1} & \textbf{LIS h@5} & \textbf{LIS PPL} \\
\midrule
base\_2ary & train & .001 & .088 & 5.29 & .882 & .995 & 1.16 \\
base\_2ary & test & .000 & .000 & 490 & .125 & .362 & 9.19 \\
base\_3ary & train & .001 & .093 & 5.22 & .871 & .995 & 1.15 \\
base\_3ary & test & .000 & .000 & 495 & .131 & .367 & 9.14 \\
derived & train & .000 & .084 & 5.25 & .873 & .996 & 1.17 \\
derived & test & .000 & .000 & 489 & .122 & .352 & 10.39 \\
\bottomrule
\end{tabular}

\vspace{0.5em}
\textbf{Aggregated}

\begin{tabular}{@{}llcccccc@{}}
\toprule
\textbf{Direction} & \textbf{Name set} & \textbf{Serialized h@1} & \textbf{Serialized h@5} & \textbf{Serialized PPL} & \textbf{LIS h@1} & \textbf{LIS h@5} & \textbf{LIS PPL} \\
\midrule
KG$\to$text & train & .000 & .089 & 5.43 & .925 & .999 & 1.15 \\
KG$\to$text & test & .000 & .000 & 513 & .135 & .301 & 11.23 \\
text$\to$KG & train & .001 & .088 & 5.25 & .875 & .995 & 1.16 \\
text$\to$KG & test & .000 & .000 & 491 & .126 & .360 & 9.57 \\
\bottomrule
\end{tabular}
\end{table}

\subsection{Template-Diverse KG-to-Text with Content-Derived Fact Addresses}\label{app:v3-template-diagnostic}

This diagnostic tests whether a small character-level Transformer can map structured KG facts to natural-language text while generalizing to unseen entity names.
It also gives a small-scale test of content-derived fact addressing.
Both models are still ordinary causal Transformers over one token sequence: the KG prefix and the text suffix form a single token sequence separated by \texttt{|}.
The difference is how token addresses are assigned inside the attention score.
The serialized baseline uses standard RoPE positions over the whole sequence.
The LIS model uses the same causal attention computation, but each KG entity-token address combines three factors throughout this diagnostic: a within-entity position rotation, a relation-slot rotation, and a content-derived fact-address rotation.
Training computes this fact address for the single KG fact in the example. Multi-fact evaluation applies the same learned rule separately to each fact in the prefix.
Thus the model attends over KG tokens and text tokens in one pass, while the pairwise score $q_i^\top P_{j\to i}k_j$ uses addresses generated from KG structure rather than addresses generated only from sequence position.

\noindent\textbf{Domains and relations.}
The data contain five domains, each with six ordered roles:
family (great-grandson, grandson, son, father, grandfather, great-grandfather),
work (intern, employee, boss, superboss, director, executive),
military (recruit, private, sergeant, captain, major, colonel),
school (freshman, sophomore, junior, senior, graduate, alumnus), and
church (acolyte, deacon, priest, bishop, archbishop, cardinal).
Each domain receives a random permutation of all 3,000 training names as a chain.
Consecutive pairs generate 25 chain2 relations, and consecutive triples generate 20 chain3 relations.
For example, in the family domain, a chain2 fact may relate \texttt{son} and \texttt{father}, while a chain3 fact may relate \texttt{grandson}, \texttt{son}, and \texttt{father}.
Chain3 is harder because the model predicts two entities instead of one and has six given/prediction permutations rather than two.
The resulting training set has 134,935 KG facts and 6.1M text sentences.

\noindent\textbf{Template diversity.}
With only one text template, a model can lean heavily on a single surface pattern.
Here each fact is paired with one of twelve sentence templates, including
\texttt{if adam is the son, the father is brian.}
and \texttt{in the family, if adam is the son, then brian is the father.}
There are 24 text variants per chain2 relation and 72 per chain3 relation.
The template words are largely predictable from the chosen template. The difficult part is the variable binding problem: mapping the KG roles to the correct text roles and spelling the corresponding entity names.

\noindent\textbf{Model formats.}
The serialized baseline sees the KG fact as a text string with role tokens, for example
\[
\texttt{<son> adam <father> brian | if adam is the son, the father is brian.}
\]
The entity order in the KG prefix is randomly shuffled each time, so the baseline cannot infer roles from KG order alone.
Standard RoPE positions run continuously across the KG prefix, separator, and text suffix.

The LIS model stores the KG side as native character tokens without role-marker tokens, for example
\[
\texttt{ adam  brian | if adam is the son, the father is brian.}
\]
Role identity is encoded by learned rotary offsets.
For a KG token in a relation slot, the angle is
\[
\texttt{position\_in\_slot}\cdot \texttt{base\_freq}
+ \texttt{slot\_angle}[\texttt{relation},\texttt{role}],
\]
where \texttt{position\_in\_slot} resets to zero for each entity.
The natural-language portion receives no learned relation or slot angle.
Both models share the same text prediction head and causal language-modeling objective.
The serialized baseline has about 151K parameters and the LIS model about 158K parameters. These counts are slightly lower than the zero-shot runs because this diagnostic uses a smaller vocabulary and role-token inventory.

\noindent\textbf{Multi-fact prefix evaluation and content-derived fact addresses.}
Training examples contain one paired KG fact and its corresponding text sentence.
At evaluation, the text sentence is still paired with its true KG fact, but the prefix may also contain four distractor KG facts before the text.
This is an intentional train/eval shift: the target relation is still present, but the model must identify it inside a longer KG prefix than it saw during training.
It tests what happens when more KG facts are placed in context than the training format used.
The single-fact setting contains one KG fact.
The \texttt{same\_entities} setting (distractors share entities but differ in relation) contains the true KG fact plus four distractors, so the model must select the KG fact whose relation is expressed by the text sentence.
The \texttt{same\_rel} setting (distractors share the relation but differ in entities) contains the true KG fact plus four distractors, so the model must bind the given text entity to the correct fact.
In both multi-fact settings, the four distractors are sampled uniformly from the qualifying facts for that setting.

The train/test address split is explicit here.
For the serialized baseline, fact identity is represented by where each fact appears in the concatenated prefix, together with separator tokens and standard positions.
Because training uses a single fact, adding distractor facts at evaluation creates fact ranges and relative offsets that were not part of the training format.
As more KG facts are placed before the text, those position-defined fact identities can move farther from the coordinate ranges seen during training.
For LIS, during both training and evaluation, each fact receives a content-derived fact address, implemented as an angle offset: rotate each KG entity token by its slot angle, mean-pool the rotated embeddings within the fact, project through LayerNorm and a linear map, and add the resulting fact-specific angle to all tokens in that fact.
This changes the rotary operator angles for KG tokens in that fact, so fact identity enters the pairwise attention operator rather than the token embedding stream.
Adding facts adds endpoints under the same addressing scheme. Fact identity never depends on a new coordinate range.
The remaining cost is competition among more attended endpoints.

\noindent\textbf{Training and evaluation.}
Training uses 3,000 random lowercase names of length 2--8 and 100 disjoint test names.
Each training example pairs one KG fact with one randomly chosen text variant and is trained causally on the full sequence.
Both models use embedding dimension 60, three layers, softmax attention, batch size 32, learning rate $5\times 10^{-4}$, and 100K training iterations.
At evaluation, template words are teacher-forced and only entity-name characters in the text are scored.
Hit@1 requires every scored character to be the argmax prediction. Hit@5 requires every scored character to be in the top five predictions. PPL is the geometric mean perplexity over scored characters only.

\Cref{tab:v3-single-fact} reports the single-fact setting.
In that setting, LIS has higher hit rates and lower PPL than the serialized baseline across the reported rows.
On chain2, LIS has a much smaller train/test gap (.972 test hit@5 versus 1.000 train), while the serialized baseline falls from .907 train hit@5 to .359 on unseen names.
On chain3, LIS reaches .923 test hit@5, while the serialized baseline falls to .126.

\begin{table}[ht]
\centering
\caption{Template-diverse KG$\to$text diagnostic, single-fact setting. The model conditions on one KG fact and predicts a text realization. Scores are over target entity-name characters only.}
\label{tab:v3-single-fact}
\scriptsize
\setlength{\tabcolsep}{4pt}
\resizebox{\columnwidth}{!}{%
\begin{tabular}{@{}llcccccc@{}}
\toprule
\textbf{Fact type} & \textbf{Set} & \textbf{Serialized h@1} & \textbf{Serialized h@5} & \textbf{Serialized PPL} & \textbf{LIS h@1} & \textbf{LIS h@5} & \textbf{LIS PPL} \\
\midrule
chain2 & train & .224 & .907 & 2.28 & .507 & 1.000 & 1.30 \\
chain2 & test  & .095 & .359 & 6.22 & .463 & .972 & 1.37 \\
chain3 & train & .011 & .743 & 2.75 & .063 & .995 & 1.46 \\
chain3 & test  & .002 & .126 & 8.64 & .050 & .923 & 1.57 \\
\bottomrule
\end{tabular}
}
\end{table}

\Cref{tab:v3-same-entities} reports the same-entities multi-fact prefix setting.
The distractors contain the same entity names as the true fact but under different relations.
The task therefore requires selecting the KG relation expressed by the text sentence.
LIS remains strong, especially on chain3 (.835 test hit@5), while the serialized baseline is near zero on chain3 test hit@5.

\begin{table}[ht]
\centering
\caption{Template-diverse KG$\to$text diagnostic, same-entities multi-fact prefix setting. The KG prefix contains the true fact and four distractors with the same entities but different relations.}
\label{tab:v3-same-entities}
\scriptsize
\setlength{\tabcolsep}{4pt}
\resizebox{\columnwidth}{!}{%
\begin{tabular}{@{}llcccccc@{}}
\toprule
\textbf{Fact type} & \textbf{Set} & \textbf{Serialized h@1} & \textbf{Serialized h@5} & \textbf{Serialized PPL} & \textbf{LIS h@1} & \textbf{LIS h@5} & \textbf{LIS PPL} \\
\midrule
chain2 & train & .215 & .853 & 2.58 & .483 & 1.000 & 1.39 \\
chain2 & test  & .133 & .411 & 6.23 & .388 & .952 & 1.57 \\
chain3 & train & .000 & .126 & 15.37 & .044 & .962 & 1.64 \\
chain3 & test  & .000 & .030 & 29.83 & .016 & .835 & 1.98 \\
\bottomrule
\end{tabular}
}
\end{table}

\Cref{tab:v3-same-rel} is the hardest multi-fact prefix setting.
All distractors have the same relation as the true fact, so the model must bind the given entity in the text to the correct KG fact.
Neither model solves this setting.
The serialized baseline slightly edges LIS on chain2 train h@5, and LIS is substantially better on chain3 PPL (18.24 versus 101.94 on test), but the low hit rates show that entity binding among same-relation distractors remains the main unresolved difficulty in this diagnostic.

\begin{table}[ht]
\centering
\caption{Template-diverse KG$\to$text diagnostic, same-relation multi-fact prefix setting. The KG prefix contains the true fact and four distractors with the same relation but different entities.}
\label{tab:v3-same-rel}
\scriptsize
\setlength{\tabcolsep}{4pt}
\resizebox{\columnwidth}{!}{%
\begin{tabular}{@{}llcccccc@{}}
\toprule
\textbf{Fact type} & \textbf{Set} & \textbf{Serialized h@1} & \textbf{Serialized h@5} & \textbf{Serialized PPL} & \textbf{LIS h@1} & \textbf{LIS h@5} & \textbf{LIS PPL} \\
\midrule
chain2 & train & .025 & .368 & 6.72 & .050 & .342 & 8.29 \\
chain2 & test  & .009 & .109 & 20.49 & .023 & .182 & 17.19 \\
chain3 & train & .000 & .005 & 51.44 & .001 & .112 & 8.58 \\
chain3 & test  & .000 & .000 & 101.94 & .001 & .038 & 18.24 \\
\bottomrule
\end{tabular}
}
\end{table}

\noindent\textbf{Serialized content-addressing control.}
We also run a control $B_{\ell}$ that uses the same serialized KG format as the baseline: facts are written as relation-marked entity strings in one flat causal sequence.
The only change is the rotary address.
The baseline $B$ uses standard continuous RoPE positions through the KG prefix and text suffix.
The $B_{\ell}$ control resets RoPE positions at segment boundaries, where a segment is one serialized KG fact or the natural-language text portion.
Each completed segment receives a content-derived angle offset: token embeddings are rotated by their within-segment RoPE angles, scatter-added by segment, mean-pooled, and projected through LayerNorm and a linear map to produce an angle offset.
This offset is added to all tokens in the segment. The final text segment uses zero offset because its content is incomplete during generation.
Within a segment, the shared offset cancels in relative angles, while cross-segment attention can use the offset to distinguish segments by content.
Thus $B$ and $B_{\ell}$ see the same data in the same format. The $B_{\ell}$ control tests whether content-derived segment addressing alone closes the gap to native LIS role/slot addresses.

\begin{table}[ht]
\centering
\caption{Template-diverse KG$\to$text diagnostic with a serialized content-addressing control. $B_{\ell}$ uses the same serialized input format as $B$, but resets positions by segment and adds a content-derived segment angle. $J$ is the native LIS model with role/slot addresses and content-derived fact addresses.}
\label{tab:v3-serialized-content-control}
\scriptsize
\setlength{\tabcolsep}{3pt}
\resizebox{\columnwidth}{!}{%
\begin{tabular}{@{}lllcccccc@{}}
\toprule
\textbf{Setting} & \textbf{Fact type} & \textbf{Set} & \textbf{$B$ h@5} & \textbf{$B_{\ell}$ h@5} & \textbf{$J$ h@5} & \textbf{$B$ PPL} & \textbf{$B_{\ell}$ PPL} & \textbf{$J$ PPL} \\
\midrule
single-fact & chain2 & train & .907 & .970 & 1.000 & 2.28 & 1.42 & 1.30 \\
single-fact & chain2 & test  & .359 & .925 & .972 & 6.22 & 1.58 & 1.37 \\
single-fact & chain3 & train & .743 & .994 & .995 & 2.75 & 1.50 & 1.46 \\
single-fact & chain3 & test  & .126 & .912 & .923 & 8.64 & 1.66 & 1.57 \\
\midrule
same-entities & chain2 & train & .853 & .973 & 1.000 & 2.58 & 1.91 & 1.39 \\
same-entities & chain2 & test  & .411 & .872 & .952 & 6.23 & 2.35 & 1.57 \\
same-entities & chain3 & train & .126 & .714 & .962 & 15.37 & 2.50 & 1.64 \\
same-entities & chain3 & test  & .030 & .525 & .835 & 29.83 & 3.35 & 1.98 \\
\midrule
same-relation & chain2 & train & .368 & .163 & .342 & 6.72 & 18.53 & 8.29 \\
same-relation & chain2 & test  & .109 & .122 & .182 & 20.49 & 41.08 & 17.19 \\
same-relation & chain3 & train & .005 & .032 & .112 & 51.44 & 28.03 & 8.58 \\
same-relation & chain3 & test  & .000 & .028 & .038 & 101.94 & 56.09 & 18.24 \\
\bottomrule
\end{tabular}
}
\end{table}

The control shows that content-derived segment addresses strongly improve the serialized model: for example, single-fact chain2 test h@5 rises from .359 to .925.
They do not close the gap to native LIS role/slot addressing.
Relative to $B_{\ell}$, $J$ has higher h@5 and lower PPL in every reported row.
Relative to the original serialized baseline, $J$ is not uniformly better in same-relation chain2 train, but the low same-relation hit rates show that this remains the hardest entity-binding setting.

Template diversity makes the difference between the two addressing schemes clearer than a single template would.
This explains why the serialized baseline scores near zero in the 1,000-name zero-shot tests but partially learns here: 3,000 names and twelve templates give it enough variable diversity to fit parts of the single-fact task, so the remaining multi-fact prefix gap is not just a failure to learn the data format.
The serialized baseline can fit parts of the single-fact chain2 task, but it has a large train/test gap and degrades sharply when fact identity must be recovered from a longer KG prefix.
The LIS model keeps role and fact information in the address operator and has much stronger unseen-name scores in the single-fact and same-entities settings.
The same-relation setting shows the remaining hard case: content-derived fact addresses improve some scores, but small three-layer models still struggle when several candidate facts share the same relation and must be separated by entity binding alone.

\section{Additional Related Work}\label{app:additional-related-work}

\noindent\textbf{Contextual position encoding.}
CoPE computes positions inside attention by a content-dependent gate: the position of one token relative to another is a learned count of intervening tokens that matter for the current query~\citep{golovneva2024cope}.
This is close to our diagnosis because the address input is no longer just the raw token position. A model can learn to count words, sentences, or other context-selected units.
The difference is that CoPE still produces a position count and then uses a position embedding for that count.
LISformer gives each completed line an instance operator rather than a larger position count.
Context growth then adds instances instead of larger coordinate values.

\noindent\textbf{Randomized positions.}
Randomized positional encodings train on randomized positions from a larger range to reduce the train/test mismatch of unseen absolute positions~\citep{ruoss2023randomized}.
Our random-line-address control is related in spirit but different in mechanism: completed text units receive independent random instance offsets on each forward pass, drawn from the same distribution at train and test.
The random-address control remains approximately constant in this setup. The LISformer row, which uses content-derived line addresses, has lower all-token PPL than the random-address control at every evaluated context length in this run.

\noindent\textbf{KG, hypergraph, and graph encodings.}
The observation that pairwise facts need shared instance identity is close to classical lossless-join and reification ideas in databases and semantic-web modeling~\citep{fagin1977multivalued,abiteboul1995foundations,noy2006defining}.
N-ary and hyper-relational KG embeddings model higher-arity facts directly for KG scoring and link prediction, including m-TransH~\citep{wen2016beyondbinary}, NaLP~\citep{guan2019nalp}, HINGE~\citep{rosso2020hinge}, HypE/HSimplE~\citep{fatemi2020hype}, and ReAlE~\citep{fatemi2023real}.
HGNN~\citep{feng2019hypergraph}, HyperGT~\citep{Liu2024HyperGT}, HGT~\citep{Hu2020HGT}, Graphormer~\citep{ying2021graphormer}, TokenGT~\citep{kim2022tokengt}, and Laplacian/random-walk graph positional encodings~\citep{dwivedi2023benchmarking} encode graph structure for graph learning.
Relational GCNs and KG-language models such as R-GCN, KG-BERT, and CoLAKE also inject graph or KG structure through graph message passing, serialized triples, or entity-aware pretraining~\citep{schlichtkrull2018rgcn,yao2019kgbert,sun2020colake}.
In particular, TokenGT turns graph nodes and edges into tokens and attaches structural identity through identifier embeddings.
LIS also tokenizes structural endpoints, but role and instance identity enter through journey operators inside the pairwise attention score rather than only as additive identifiers.

\noindent\textbf{Geometric operator views.}
Gauge-equivariant geometric models use local frames and parallel transport on manifolds or meshes~\citep{cohen2019gauge,dehaan2021gauge,he2021gauge}. LIS uses the same broad operator-product idea for role-labeled data.
The address-factorization result is the role/instance version of a standard relative-measurement question~\citep{singer2011angular}: when do local relative measurements determine global addresses, and what freedom remains in choosing the reference frame?

\section*{AI Use Statement}
The authors used AI assistants for language editing, LaTeX assistance, and feedback on manuscript presentation.
All mathematical claims, proofs, citations, and final text were checked and approved by the authors, who take responsibility for the submission.


\end{document}